\documentclass{article}

\usepackage{wrapfig}

\usepackage[final,nonatbib]{neurips_2025}

\usepackage{microtype} 

\usepackage[utf8]{inputenc} 
\usepackage[T1]{fontenc}    
\usepackage{hyperref}       
\usepackage{url}            
\usepackage{booktabs}       
\usepackage{amsfonts}       
\usepackage{nicefrac}       

\usepackage{xcolor}

\usepackage{algorithm}
\usepackage{algpseudocode}
\usepackage{amsmath}
\usepackage{amssymb}
\usepackage{amsthm}
\usepackage{mathtools}
\usepackage{braket}
\usepackage[utf8]{inputenc}
\usepackage{colortbl}
\usepackage{enumitem}
\usepackage{lipsum}
\usepackage{wrapfig}
\usepackage[square,numbers]{natbib}

\algnewcommand{\LineComment}[1]{\State $\triangleright$ #1}

\DeclareMathOperator*{\R}{\mathbb{R}}

\DeclareMathOperator{\E}{\mathbb{E}}
\DeclareMathOperator*{\cO}{\mathcal{O}}

\DeclareMathOperator{\TV}{\mathrm{TV}}

\newtheorem{assumption}{Assumption}

\newtheorem{definition}{Definition}
\newtheorem{theorem}{Theorem}

\newtheorem{lemma}{Lemma}
\newtheorem{remark}{Remark}

\newcommand{\norm}[1]{\left \lVert #1 \right\rVert }

\title{Regret Analysis of Average-Reward Unichain MDPs via an Actor-Critic Approach
}
\author{
  Swetha Ganesh  \\
  Purdue University, USA \\
  \texttt{ganesh49@purdue.edu} \And
  Vaneet Aggarwal \\
  Purdue University, USA \\
  \texttt{vaneet@purdue.edu} \\
}

\begin{document}

\maketitle

\begin{abstract}
Actor-Critic methods are widely used for their scalability, yet existing theoretical guarantees for infinite-horizon average-reward Markov Decision Processes (MDPs) often rely on restrictive ergodicity assumptions. We propose NAC-B, a Natural Actor-Critic with Batching, that achieves order-optimal regret of $\tilde{O}(\sqrt{T})$ in infinite-horizon average-reward MDPs under the unichain assumption, which permits both transient states and periodicity. This assumption is among the weakest under which the classic policy gradient theorem remains valid for average-reward settings. NAC-B employs function approximation for both the actor and the critic, enabling scalability to problems with large state and action spaces. The use of batching in our algorithm helps mitigate potential periodicity in the MDP and reduces stochasticity in gradient estimates, and our analysis formalizes these benefits through the introduction of the constants $C_{\text{hit}}$ and $C_{\text{tar}}$, which characterize the rate at which empirical averages over Markovian samples converge to the stationary distribution.
\end{abstract}

\section{Introduction}
%Structure: (i) Background of average reward RL, PG (ii) overview current literature in these areas (iii) research gap and question (iv) overview of our approach (v) main contributions and challenges (provided pg theorem, perf difference etc. that was not known previously, exponentially fast to mixing time result not there anymore).

Reinforcement Learning (RL) involves an agent interacting with an unknown environment to maximize long-term rewards. It has been successfully applied to diverse areas such as traffic engineering, resource allocation, and ride-sharing \citep{geng2020multi, chen2024learning, al2019deeppool}. RL problems are commonly framed as episodic, discounted, or (infinite-horizon) average-reward; the average-reward setting is especially suited to real-world tasks due to its ability to better capture long-term behavior.

%In the average-reward Markov Decision Process (MDP) setting, which is the focus of our work, a central measure of algorithmic performance is the expected regret. It has been shown in \citep{auer2008near} that for a broad class of MDPs, the expected regret of any algorithm is lower bounded by $\Omega(\sqrt{T})$, where $T$ is the time horizon. However, several regret analyses in the average reward setup assume ergodicity which greatly simplifies the analysis due to the exponentially fast mixing to the stationary distribution and is hard to verify whether such assumptions hold in practice. 

Our focus is on the average-reward Markov Decision Process (MDP) setting, where a key measure of algorithmic performance is expected regret. \citep{auer2008near} established that for a broad class of MDPs, any algorithm incurs expected regret lower bounded by $\Omega(\sqrt{T})$, where $T$ is the time horizon. Many existing regret analyses in this setting assume ergodicity, a strong assumption that simplifies analysis by ensuring fast mixing to a stationary distribution, but is often difficult to verify or justify in practice.

In the absence of ergodicity, existing algorithms are typically designed for model-based or tabular settings, with computational complexity that scales with $|\mathcal{S}|$ and $|\mathcal{A}|$, the cardinalities of the state and action spaces, respectively. These requirements are prohibitive for environments with large or continuous state and action spaces. Although linear MDPs reduce improves dependence on $|\mathcal{S}|$, they still rely on strong structural assumptions, namely, linearity of the transition dynamics and reward functions. Moreover, value-based methods in such settings often remain computationally expensive due to the need to maximize over all actions at each iteration.

%In regret analyses without ergodicity, these algorithms are specifically designed for either the model-based setting, where the space complexity is $\Omega(|\mathcal{S}|^2|\mathcal{A}|)$ or the tabular setting where this complexity goes down to $\Omega(|\mathcal{S}||\mathcal{A}|)$, but still remains challenging in applications with large or infinite state or action spaces. In the case of linear MDPs, this complexity can be reduced; however, the approach still relies on the strong assumption of linearity of the transition matrix and reward function, and value-based methods in such settings continue to suffer from high computational costs in large action spaces, as the algorithm requires maximizing over all actions at each iteration.

%To address these limitations, one promising direction is to move beyond value-based methods to a Policy Gradient (PG) approach. In this approach, policies are represented by parameter vectors, often using neural networks, and the learning process involves updating these parameters using gradient descent. 

To overcome these limitations, a promising alternative is the use of Policy Gradient (PG) methods, which directly optimize parameterized policies, often modeled by neural networks, via gradient descent. While PG methods have shown strong empirical performance, most theoretical analyses assume ergodicity. In contrast, our work develops PG methods that do not rely on ergodicity. Specifically, we make only a minimal assumption: each policy induces a unichain MDP, i.e., it has a single recurrent class. This relaxation enables analysis in more general and realistic environments.

%Despite their empirical success, most theoretical analyses of PG methods rely on the assumption of ergodicity. In contrast, our work investigates policy gradient methods beyond the ergodic regime. Specifically, we do not assume irreducibility or aperiodicity. Instead, we adopt a minimal assumption: each policy induces a unique recurrent class. 

%Under this general setting, we establish several foundational results, including a generalized policy gradient theorem, a gradient domination lemma, and novel bounds on the value function. Deriving these results is particularly challenging due to the absence of exponentially fast mixing, which typically serves as a key tool in analyses under the ergodic assumption.

\subsection{Related works}

\begin{table*}[t]\label{tab:regret}
\begin{center}
\footnotesize
\begin{tabular}{|c|c|c|c|c|}
\hline
\textbf{ Algorithm} & \textbf{Regret}  & \textbf{Ergodicity-free} & \textbf{General Policy }\\ \hline

  MDP-OOMD {\small\citep{wei2020model}}& $\tilde{O}(\sqrt{T})$ & No & No \\
     \hline

   Optimistic Q-learning {\small\citep{wei2020model}}&  $\tilde{O}(T^{2/3})$ &Yes$^{\color{blue}(1)}$& No \\
      \hline
    MDP-EXP2 {\small\citep{wei2021learning}}& $\tilde{O}(\sqrt{T})$ &  No & No \\
     \hline

 UCB-AVG \citep{pmlr-v195-zhang23b}& $\tilde{O}(\sqrt{T})$& Yes$^{\color{blue}(1)}$ & No\\
    \hline

  PPG \citep{bai2023regret}& $\tilde{O}(T^{3/4})$ & No & Yes\\
     \hline

   PHAPG \citep{ganesh2024variance}& $\tilde{O}(\sqrt{T})$ & No & Yes\\
      \hline
   Optimistic Q-learning {\small\citep{agrawal2025optimistic}} &  $\tilde{O}(\sqrt{T})$ & Yes& No \\
   \hline
   $\gamma$-DC-LSCVI-UCB \citep{hong2025comp}& $\tilde{O}(\sqrt{T})$& Yes$^{\color{blue}(1)}$ & No\\
   \hline
    \rowcolor{green!25} This work (Algorithm \ref{alg:acb})& $\tilde{O}(\sqrt{T})$& Yes & Yes\\
   \hline
\end{tabular}
\caption{A comparison of regret bounds for model-free RL algorithms in infinite-horizon average-reward MDPs. By "General Policy," we refer to algorithms that employ a parameterized, policy-based approach where the parameter vector $\theta \in \R^d$, possibly with $d\ll|\mathcal{S}||\mathcal{A}|$. $^{\color{blue}(1)}$ These analyses consider the more general setting of weakly communicating MDPs or Bellman optimality. \vspace{-5mm}
}
\end{center}
\end{table*}

\textbf{Value-Based Approaches:} Model-based algorithms like those in \citep{auer2008near,agrawal2017optimistic,jaksch2010near,zurek2024spanbased}, as well as recent model-free methods, achieve the optimal $\mathcal{O}(\sqrt{T})$ regret. However, most are tabular and value-based, maintaining Q-values for each state action pair, an approach that scales poorly to large or continuous spaces. Examples include optimistic Q-learning methods \citep{wei2020model, agrawal2025optimistic, pmlr-v195-zhang23b}. The $\gamma$-DC-LSCVI-UCB algorithm \citep{hong2025comp} improves scalability with respect to $|\mathcal{S}|$ but requires maximization over the entire action set at each iteration, which is still computationally expensive for large $|\mathcal{A}|$. Furthermore, its analysis assumes a linear MDP with known feature representations, whereas our setting is considerably more general and does not rely on such structural knowledge or assumptions.

\textbf{Direct Policy Gradient Approaches:} Unlike value-based methods, policy gradient (PG) approaches are well-suited to environments with large or continuous state and action spaces, and are relatively easy to implement. Despite their empirical success, theoretical analysis of PG methods has largely focused on ergodic MDPs. This is due to favorable properties in the ergodic setting, such as bounded hitting times and fast mixing, which simplify analysis and enable near-optimal regret bounds. These properties ensure that samples collected $t_{\mathrm{mix}}$ steps apart approximate independent draws from the stationary distribution. Algorithms like MDP-OOMD \citep{wei2020model}, PPG \citep{bai2023regret}, and PHAPG \citep{ganesh2024variance} exploit these characteristics to construct low-bias value estimates. However, to the best of our knowledge, no prior work provides regret guarantees for PG methods in the more general unichain setting, where such strong mixing assumptions do not hold.

%In contrast to value-based methods, policy gradient (PG) approaches have gained popularity in practice due to their natural compatibility with large or continuous state and action spaces and their ease of implementation. Despite their empirical success, the theoretical analysis of PG methods remains largely limited to ergodic MDPs. This restriction stems from the fact that ergodic MDPs exhibit beneficial properties such as bounded worst-case hitting times and exponentially fast mixing. These properties are instrumental in simplifying the analysis and enabling order-optimal regret guarantees, as they ensure that samples collected $t_{\mathrm{mix}}$ steps apart are approximately independent and drawn from the stationary distribution. For instance, algorithms such as MDP-OOMD \citep{wei2020model}, PPG \citep{bai2023regret}, and PHAPG in \citep{ganesh2024variance} exploit this mixing behavior to construct low-bias estimates of the value function.

\textbf{Actor-Critic Approaches:} In contrast to direct policy gradient methods that estimate gradients from sampled trajectories, Actor-Critic (AC) methods use Temporal Difference (TD) based critics to aid policy gradient through bootstrapped value estimates. Further, direct methods scale poorly and rely on access to mixing or hitting times, typically unknown in practice. AC methods are more sample-efficient but introduce addtional bias, making theoretical analysis harder. Regret results for AC methods are limited and focus on global convergence (pseudo-regret), as seen in \citep{patel2024global, wang2024nonasymptotic, ganesh2024order}. The MLMC-NAC algorithm in \citep{ganesh2024order} achieves order-optimal convergence using multi-level Monte Carlo to reduce bias, but still assumes fast mixing. Extending such results to unichain MDPs, which lack exponential mixing, requires fundamentally different techniques and is the focus of this paper.

%We note that the above mentioned works use direct policy gradient approaches, which estimates gradients directly from sampled trajectories. In contrast, the actor-critic (AC) method leverages a temporal-difference (TD) based critic to estimate the value function and guide policy updates. While direct methods are often easier to analyze and have been the focus of most existing regret analyses under general policy parameterizations, they suffer from scalability issues and rely heavily on knowledge of mixing or hitting times—quantities that are generally unavailable in real-world applications. In contrast, actor-critic methods avoid these limitations by using bootstrapped estimates, which are more sample-efficient in practice but introduce bias and complicate the theoretical analysis.

%Regret analyses for actor-critic methods are still relatively scarce and typically focus on global convergence rates, also referred to as pseudo-regret. Global convergence results for AC are provided in \citep{patel2024global,wang2024nonasymptotic,ganesh2024order}. In particular, an order-optimal result in this setting was recently established for MLMC-NAC in \citep{ganesh2024order}, using a multi-level Monte Carlo technique to mitigate bias, an approach that again hinges on fast mixing. Extending such results to unichain MDPs, which do not necessarily mix exponentially fast, demands considerable departures in both algorithmic design and analysis. 

\textbf{Unichain Analyses}: A few recent works explicitly address the unichain setting \citep{agrawal2025optimistic, li2024stochastic}. The Optimistic Q-learning algorithm in \citep{agrawal2025optimistic}, discussed earlier, is value-based. The SPMD method in \citep{li2024stochastic} takes a policy-based approach with order-optimal sample complexity but uses a tabular policy, leading to poor scalability with large action spaces. It also assumes access to a simulator (generator), which is often unrealistic. Moreover, both works rely on stronger assumptions than the classical unichain condition used in our analysis (see Remark~\ref{rem:unichain}).

\textbf{Discounted to Average Techniques}: The most common strategy for relaxing ergodicity assumptions in reinforcement learning is to reduce the average-reward problem to a discounted MDP. This reduction is widely adopted in works addressing weakly communicating or Bellman optimality settings \citep{wei2020model,pmlr-v195-zhang23b,hong2025comp,zurek2024spanbased}, where the discounted problem is solved with a discount factor $\gamma$ close to 1. In particular, these works select $\gamma$ satisfying $1/(1-\gamma) \simeq T^\beta$ for some $\beta>0$. Notably, this approach requires very sharp bounds in DMDPs in terms of $1/(1-\gamma)$ whereas, existing guarantees for policy gradient methods in the discounted setting exhibit very poor dependence on this term, resulting in much weaker performance bounds if such Discounted to Average Techniques are used (see Appendix~\ref{app:disc-avg}). Additionally, these approaches require access to a simulator that allows sampling from a given distribution $\rho$ at each iteration, which  can be impractical. %Consequently, such approaches remain largely impractical despite their theoretical appeal.

\subsection{Main contributions}

We propose an actor-critic algorithm in Algorithm~\ref{alg:acb}, Natural Actor Critic with Batching (NAC-B), that achieves order-optimal regret of $\tilde{O}(\sqrt{T})$ in average-reward Markov decision processes (MDPs) under the unichain assumption. Our key contributions are summarized as follows:

\begin{itemize}[leftmargin=*]
    \item {\bf First Regret Guarantees for General Policy Gradients Beyond Ergodicity:} While policy gradient (PG) methods with general parameterizations have been analyzed under the restrictive ergodicity assumption, we extend this to the weaker unichain condition. This aligns with the assumption used in the foundational policy gradient theorem for average-reward MDPs \citep{sutton1999policy}. Our work provides the first regret guarantees for parameterized PG methods under this setting.

  %  We provide the first regret guarantees for policy gradient methods with general parameterizations in the average-reward setting. Unlike value-based approaches requiring action maximization or tabular storage, our algorithm uses parameterized actors and critics with fixed-dimensional representations, ensuring scalability to large state/action spaces.
    
    %Our algorithm builds on policy gradient methods, which are widely used in practice due to their scalability and efficiency, unlike Q-learning–based methods that require costly maximization over the action space at each step. Popular methods such as PPO and TRPO follow this paradigm. We provide the first optimal regret guarantees in the average-reward setting for such parametrized actor-critic algorithms, using a fixed-dimensional representation for both the actor and the critic. This ensures that the algorithm’s memory and computational complexity remain independent of the sizes of the state and action spaces.

%\item Extension to the unichain setting: In contrast to most existing parametric actor-critic algorithms, which are developed under the assumption of ergodicity, our analysis is conducted under the more general unichain assumption. This is the weakest condition under which the average-reward function is independent of the initial-state distribution. This assumption is also used in the foundational policy gradient theorem for average-reward MDPs \citep{sutton1999policy}, and it remains unknown whether a similar result holds without it.

\item {\bf Online Learning Without Simulator Access:} Unlike discounted-MDP reductions that require a simulator to reset to arbitrary initial distributions at each iteration, our algorithm operates entirely online in unichain MDPs. This enables learning in environments where simulator access is unavailable or costly (e.g., real-time network traffic engineering, robotics).

\item {\bf Technical Challenges:} Analyses in the ergodic setting rely crucially on the property of exponentially fast mixing. In contrast, the classical unichain setting does not generally admit analogous results. For example, in a periodic Markov decision process, the quantity $\lim_{t\to\infty} (P^{\pi_\theta})^t(s_0, \cdot)$ may not even exist. To address this challenge arising from periodicity, we employ large-batch averaging. Specifically, we show that the time-averaged distribution $
\left\| \frac{1}{t}\sum_{i=1}^t(P^{\pi_\theta})^i(s_0, \cdot) - d^{\pi_\theta}(\cdot) \right\|_{\mathrm{TV}}
$ converges to the stationary distribution at a rate of $\cO(1/t)$ (see Lemma \ref{lem:ergodic-avg}). This averaging approach enables us to show that the bias and variance of our estimators decay at a sufficiently fast rate, albeit not exponentially (see Lemma \ref{lem:bias-variance-unichain}). An additional consequence of these results is that they facilitate the derivation of bounds on the value and $Q$-functions.

Despite the reduction in bias and variance, transient states continue to pose analytical challenges. Intuitively, this is because they do not contribute meaningful long-term information. A more technical explanation is provided in Section~\ref{subsec:main-result}. To isolate the impact of transient states, we prove a rapid entry into the recurrent class (Lemma~\ref{lem:prob-exponential}). Combined with the strong Markov property, this allows us to restrict our analysis to the recurrent class.

%The presence of transient states in the unichain setting complicates the control of bias and variance, as these states can significantly affect the early dynamics of the algorithm while offering no meaningful long-term information. We develop a new analytical approach that includes a careful decomposition of the stochastic noise, bounding the probability of entering the recurrent set within a finite number of iterations, and leveraging the strong Markov property. These tools allow us to isolate the transient phase and establish clean recursive bounds that yield sharp finite-time performance guarantees.

%\item {\bf Batch Averaging for Bias Reduction:} Our algorithm incorporates batch averaging to stabilize learning and improve estimator quality. This addresses non-stationary dependencies in unichain MDPs, where traditional mixing-time arguments fail.
%To mitigate the impact of the non-ergodic, Markovian noise in the updates arising from both the stochastic nature of sample trajectories and structural properties of the MDP such as periodicity, our algorithm incorporates a batch averaging scheme. This averaging helps stabilize learning and improve the quality of value estimates.
\end{itemize}

\section{Problem Formulation and Preliminaries}

We study an infinite-horizon reinforcement learning problem with an average reward criterion, modeled as a Markov Decision Process (MDP). The MDP is represented by the tuple $\mathcal{M} = (\mathcal{S}, \mathcal{A}, r, P, \rho)$, where $\mathcal{S}$ denotes the finite state space and $\mathcal{A}$ represents the finite action space. The reward function $r: \mathcal{S} \times \mathcal{A} \to [0,1]$ assigns a bounded reward to each state-action pair. The state transition function $P: \mathcal{S} \times \mathcal{A} \to \Delta^{|\mathcal{S}|}$ determines the probability distribution over the next state given the current state and action, where $\Delta^{|\mathcal{S}|}$ denotes the probability simplex over $\mathcal{S}$. The initial state distribution is given by $\rho: \mathcal{S} \to [0,1]$. A policy $\pi: \mathcal{S} \to \Delta^{|\mathcal{A}|}$ specifies a probability distribution over actions for each state. Given a policy $\pi$, the long-term average reward is defined as  
\begin{equation}
    J_{\rho}^{\pi} \triangleq \lim\limits_{T\rightarrow \infty}\frac{1}{T}\E\bigg[\sum_{t=0}^{T-1}r(s_t,a_t) \bigg| s_0 \sim \rho \bigg],
\end{equation}  
where the expectation is taken over trajectories generated by executing actions according to $a_t \sim \pi(\cdot | s_t)$ and transitioning states via $s_{t+1} \sim P(\cdot | s_t, a_t)$, for all $t \geq 0$. For simplicity, we drop the dependence on $\rho$ whenever it is clear from the context.  We focus on a parametrized policy class $\Pi$, where each policy is indexed by a parameter $\theta \in \Pi_\Theta$, with $\Pi_\Theta \subset \mathbb{R}^{d}$. Our goal is to solve the optimization problem:  $    \max_{\theta\in\Pi_\Theta} ~ J^{\pi_{\theta}} \triangleq J(\theta).$

A policy $\pi_{\theta} \in \Pi$ induces a transition function $P^{\pi_{\theta}}: \mathcal{S} \to \Delta^{|\mathcal{S}|}$, given as  $
    P^{\pi_{\theta}}(s, s') = \sum_{a\in\mathcal{A}} P(s' | s, a) \pi_{\theta}(a | s), \quad \forall s, s' \in \mathcal{S}$. The corresponding stationary distribution is defined as:

\begin{definition}  
    Let $d^{\pi_{\theta}} \in \Delta^{|\mathcal{S}|}$ denote the stationary distribution of the Markov chain induced by $\pi_{\theta}$, given by  
\begin{equation}
    d^{\pi_{\theta}}(s) = \lim_{T\rightarrow \infty} \frac{1}{T} \sum_{t=0}^{T-1} \mathrm{Pr}(s_t = s \mid s_0 \sim \rho, \pi_{\theta}).
\end{equation}  
\end{definition}  

We assume the following property for the MDP, which is defined as a unichain MDP:

\begin{assumption}  
    \label{assump_mdp}  
    The MDP $\mathcal{M}$ is such that, for every policy $\pi \in \Pi$, the induced Markov chain has a single recurrent class. 
\end{assumption}  

%that the Markov chain induced by each policy $\pi$ has a single recurrent class. Under this assumption, there exists a unique stationary distribution, independent of $\rho$.  More formally, we make the following assumption about the structure of the MDP. 

\begin{remark} \label{rem:unichain} An MDP $\mathcal{M}$ satisfying this property is referred to as a unichain MDP \citep{puterman2014markov}. This assumption does not require irreducibility, as transient states may be present, nor does it impose aperiodicity. Consequently, it is strictly weaker than the standard ergodicity assumptions.

Under this condition, the stationary distribution $d^{\pi_{\theta}}$ is well-defined, independent of the initial distribution $\rho$, and satisfies the balance equation:
$ (P^{\pi_{\theta}})^{\top} d^{\pi_{\theta}} = d^{\pi_{\theta}}. $

This assumption forms the basis of the foundational policy gradient theorem for average-reward MDPs \citep{sutton1999policy}. Although various alternative definitions of unichain MDPs appear in the literature, the notion employed here is classical and, in fact, weaker than several recent formulations. For example, \citep{agrawal2025optimistic} assume the existence of a state $s_0$ that is recurrent under all policies, a stronger condition than ours. Similarly, \citep{li2024stochastic} adopt a mixing unichain assumption, which additionally requires aperiodicity. In contrast, our analysis does not depend on either of these stronger conditions.
\end{remark}

%The following result provides a bound on the convergence rate of the empirical state distribution to the stationary distribution. 

% \begin{lemma}[Corollary 3, \citep{Roberts01011997}]  
%     Let Assumption \ref{assump_mdp} hold and consider a policy $\pi \in \Pi$. Let $\E_s[T_{s'}]$ denote the expected time to reach state $s'$ starting from state $s$. Then, the following bound holds:  
% \begin{equation}
% \label{eq:rosenthal-lem}
%     \norm{\frac{1}{t} \sum_{i=1}^{t} (P^{\pi})^i(s, \cdot) - d^{\pi}(\cdot)}_{\TV} \leq \frac{C_{\pi}}{t}, \quad \forall s \in \mathcal{S},
% \end{equation}  
% where $C_{\pi} \coloneqq \sum_{s' \in \mathcal{S}} d^{\pi}(s')\E_s[T_{s'} ]$,  is the expected time, starting from the point $s \in \mathcal{S}$, to hit a point $s' \in \mathcal{S}$ chosen according to $d^{\pi}(\cdot)$; furthermore, $C_{\pi}$ does not depend on $s$.
% \end{lemma}  

Since we cannot rely on the mixing time commonly used in the analysis of ergodic MDPs, we consider two alternative quantities, $C_{\mathrm{hit}}^{\theta}$ and $C_{\mathrm{tar}}^{\theta}$, which are more appropriate in the unichain setting. These are defined with respect to the Markov chain induced by a stationary policy $\pi$ as follows:

Let $\mathcal{S}^{\theta}_R \subseteq \mathcal{S}$ denote the recurrent class under policy $\pi_{\theta}$, and let $T_{\theta} \coloneqq \inf\{t \geq 0 : s_t \in \mathcal{S}^{\theta}_R\}$ denote the first hitting time of the recurrent class. Then we define
\begin{align}
\label{eq:C-hit-def}
    C_{\mathrm{hit}}^{\theta} \coloneqq \max_{s \in \mathcal{S}} \mathbb{E}_s^{\theta}[T_{\theta}],
\end{align} 
as the worst-case expected time to enter the recurrent class when starting from any state $s \in \mathcal{S}$. Note that if there are no transient states, then $C_{\mathrm{hit}}^{\theta} =0$.

Similarly, for any $s, s' \in \mathcal{S}$, let $T_{s'} \coloneqq \inf\{t \geq 0 : s_t = s'\}$ be the first hitting time to state $s'$. Then we define
  \begin{align}
  \label{eq:C-tar-def}
       C_{\mathrm{tar}}^{\theta} \coloneqq \sum_{s' \in \mathcal{S}} d^{\pi}(s') \mathbb{E}_s^{\theta}[T_{s'}],
  \end{align}
  as the expected time to reach a state drawn from $d^{\pi_{\theta}}$, starting from a fixed state $s\in \mathcal{S}^{\theta}_R$. Notably, by the Random Hitting Target lemma (Corollary 2.14, \citep{aldous2002reversible}), this quantity is independent of the choice of $s$. We define $C_{\mathrm{hit}} \coloneqq \sup_{\theta} C_{\mathrm{hit}}^{\theta}$ and $C_{\mathrm{tar}} \coloneqq \sup_{\theta} C_{\mathrm{tar}}^{\theta}$ and $C\coloneqq C_{\mathrm{hit}}+C_{\mathrm{tar}}$.

%These quantities serve as natural surrogates for mixing time in the analysis of unichain MDPs, where the chain may contain transient states and lack aperiodicity, but still admits a well-defined stationary distribution under each policy.
We write the average reward as:
\begin{equation}
    \label{eq_r_pi_theta}
        J(\theta) = \E_{s\sim d^{\pi_{\theta}}, a\sim \pi_{\theta}(\cdot|s)}[r(s, a)] = (d^{\pi_{\theta}})^\top r^{\pi_{\theta}}  \text{, where}~r^{\pi_{\theta}}(s) \triangleq \sum_{a\in\mathcal{A}}r(s, a)\pi_{\theta}(a|s), ~\forall s\in \mathcal{S}
\end{equation}
The average reward $J(\theta)$ is also independent of the initial distribution, $\rho$. Furthermore, $\forall \theta\in\Pi_\Theta$, there exist a function $Q^{\pi_{\theta}}: \mathcal{S}\times \mathcal{A}\rightarrow \mathbb{R}$ such that the following Bellman equation is satisfied $\forall (s, a)\in\mathcal{S}\times\mathcal{A}$. 
\begin{equation}
    \label{eq_bellman}
    Q^{\pi_{\theta}}(s,a)=r(s,a)-J(\theta)+\E_{s'\sim P(\cdot|s, a)}\left[V^{\pi_{\theta}}(s')\right]
\end{equation}
where the state value function, $V^{\pi_{\theta}}:\mathcal{S}\rightarrow \mathbb{R}$ is defined as    $V^{\pi_{\theta}}(s) = \sum_{a\in\mathcal{A}}\pi_{\theta} (a|s)Q^{\pi_{\theta}}(s, a), ~\forall s\in\mathcal{S}$ \citep{puterman2014markov}.  Note that if $(\ref{eq_bellman})$ is satisfied by $Q^{\pi_{\theta}}$, then it is also satisfied by $Q^{\pi_{\theta}}+c$ for any arbitrary constant, $c$.

Additionally, we define the advantage function as $
    A^{\pi_{\theta}}(s, a) \triangleq Q^{\pi_{\theta}}(s, a) - V^{\pi_{\theta}}(s)$. If Assumption~\ref{assump_mdp} holds, it is known that the policy gradient at $\theta$, $\nabla_{\theta}J(\theta)$ can be expressed as follows \citep{sutton1999policy}:
     \begin{align}
       \nabla_\theta J(\theta) = \E_{(s,a) \sim \nu^{\pi_{\theta}}} [A^{\pi_{\theta}}(s,a)\nabla_\theta \log \pi_{\theta}(a|s)],
   \end{align}
where 
$\nu^{\pi_{\theta}}(s,a) \coloneqq d^{\pi_{\theta}}(s)\pi_{\theta}(a|s)$.
This is commonly referred to as the Policy Gradient theorem and while the proof can be found in \citep{sutton1999policy}, we include it in Appendix~\ref{sec:PG:proof} for completeness.

Natural Policy Gradient (NPG) methods update $\theta$ along the NPG direction $\omega^*_\theta$ defined as
    \begin{align}
    \label{eq_exact_npg}
        \omega^*_{\theta} = F(\theta)^{\dagger} \nabla_{\theta} J(\theta),
    \end{align}
    where $\dagger$ denotes the Moore-Penrose pseudoinverse and $F(\theta)$ is the Fisher matrix defined as 
    \begin{align}
        F(\theta) = \E \left[\nabla_{\theta}\log\pi_{\theta}(a|s) \nabla_{\theta}\log\pi_{\theta}(a|s)^\top \right]
    \end{align}
where the expectation is taken over ${(s, a)\sim \nu^{\pi_{\theta}}}$. This yields the NPG update
$$ \theta_{k+1} = \theta_k + \alpha \omega_k^*.$$

\begin{wrapfigure}{r}{0.51\textwidth}
  \vspace{-23pt}
  \begin{center}
  \fbox{%
    \begin{minipage}{0.52\textwidth}
    \vspace{-.21in}
 \begin{algorithm}[H]
    \caption{Natural Actor-Critic with Batching}
    \label{alg:acb}
    \begin{algorithmic}[1]
        \State \textbf{Input:} Initial parameters $\theta_0$, $\{\omega_0^k\}$, and $\{\xi_0^k\}$, policy stepsize $\alpha$, critic parameters ($\beta$, $c_\beta$), NPG stepsize $\gamma$, initial state $s_0 \sim \rho$, outer loop size $K$, inner loop size $H$, batch-size $B$
    %    \State \textbf{Initialization:} $T \gets 0$

        \For{$k = 0, 1, \cdots, K-1$}
            \LineComment{Average reward and critic estimation}
            \For{$h = 0, 1, \cdots, H-1$}
                \State $s_k^0 \gets s_0$
                \For{$b = 1, 2, \cdots, B$}
                    \State Take action $a_b^{kh} \sim \pi_{\theta_k}(\cdot | s_b^{kh})$
                    \State Collect next state $s_{b+1}^{kh} \sim P(\cdot | s_b^{kh}, a_b^{kh})$
                    \State Receive reward $r(s_b^{kh}, a_b^{kh})$
                \EndFor
                \State Update the combined average reward and critic estimate $\xi_h^k = [\eta_h^k,\zeta_h^k]^\top$ using \eqref{eq:critic_reward_update}-\eqref{eq_def_Az_matrix}
                \State $s_0 \gets s_B^{kh}$
            \EndFor
            \State $\xi_k \gets \xi_H^k$
             \LineComment{Natural Policy Gradient estimation}
            \For{$h = 0, 1, \cdots, H-1$}         
                \State $s_k^0 \gets s_0$
                \For{$b = 1, 2, \cdots, B$}
                    \State Take action $a_b^{kh} \sim \pi_{\theta_k}(\cdot | s_b^{kh})$
                    \State Collect next state $s_{b+1}^{kh} \sim P(\cdot | s_b^{kh}, a_b^{kh})$
                    \State Receive reward $r(s_b^{kh}, a_b^{kh})$
                \EndFor
                \State Update NPG estimate $\omega_h^k$ using \eqref{eq:adv_est}-\eqref{eq:ahatu}
                \State $s_0 \gets s_B^{kh}$
            \EndFor
            \LineComment{Policy update}
            \State $s_k^0 \gets s_0$
            \State Update policy parameter $\theta_k$ using $\theta_{k+1} \gets \theta_k + \alpha \omega_k$%\eqref{eq:policy_update}
            \State $s_0 \gets s_B^k$
        \EndFor
    \end{algorithmic}
\end{algorithm}
  \vspace{-.21in}
 \end{minipage}
   }
  \end{center}
  \vspace{-80pt}
\end{wrapfigure}

Let $J^* \triangleq \sup_{\boldsymbol{\theta}\in\Pi_\Theta} J(\theta)$. For a given MDP $\mathcal{M}$ and a time horizon $T$, the regret of an algorithm $\mathbb{A}$ is defined as follows.
\begin{align}
    \mathrm{Reg}_T(\mathbb{A}, \mathcal{M}) \triangleq \sum_{t=0}^{T-1} \left(J^*-r(s_t, a_t)\right)
\end{align}
where the action, $a_t$, $t\in\{0, 1, \cdots \}$ is chosen by following the algorithm $\mathbb{A}$ based on the trajectory up to time, $t$, and the state, $s_{t+1}$ is obtained by following the state transition function, $P$. Wherever there is no confusion, we shall simplify the notation of regret to $\mathrm{Reg}_{T}$. %The goal of maximizing $J(\cdot)$ can be accomplished by designing an algorithm that minimizes the regret.

\section{Proposed Algorithm}

%We use a policy gradient based algorithm. We first provide the updates for the genie-aided case the policy gradients and the average reward can be noiselessly estimated. Then, we include the sampling to remove this assumption and have the proposed algorithm in Algorithm \ref{alg:acb}. 

%\textbf{ Swetha, In the next para - for Appendix A, write a line or two in case there is any novelty in the Appendix for unichain - then mention here. }

We propose a \emph{Natural Actor–Critic with Batching} algorithm (Algorithm~\ref{alg:acb}), which runs for $K$ outer iterations (or \emph{epochs}) of natural policy gradient updates. Each outer iteration includes $H$ inner iterations to estimate the average reward and value function, followed by another $H$ iterations to estimate the NPG direction. We first present the necessary preliminaries, followed by the detailed description of the algorithm.

\textbf{Preliminaries}: For any fixed $\theta_k$, it is known that $\omega_k^*$ is the solution to the optimization problem

$$
\omega_k^* = \arg \min_{\omega} L_{\nu^{\pi_{\theta_k}}}(\omega, \theta_k),
$$

where $L_{\nu^{\pi_\theta}}(\omega, \theta)$ is defined as follows
\begin{align}
\label{eq_def_L_nu}
    &L_{\nu^{\pi_\theta}}(\omega, \theta) \nonumber\\
    &= \dfrac{1}{2}\E\left[\big(A^{\pi_{\theta}}(s,a)-\omega^\top\nabla_{\theta}\log\pi_{\theta}(a| s)\big)^2\right],
\end{align}

where the expectation is taken over ${(s, a)\sim \nu^{\pi_{\theta}}}$. It can be seen that $L_{\nu^{\pi_\theta}}(\omega, \theta)$ is a quadratic function in $\omega$ with gradient given by $F(\theta) \omega - \nabla J(\theta)$. Since neither $F(\theta_k)$ nor $\nabla J(\theta_k)$ is available in closed form, we estimate them from samples and perform stochastic gradient descent on $L_{\nu^{\pi_{\theta_k}}}(\omega, \theta_k)$ to obtain an approximate NPG direction $\omega_k$.

We estimate $F(\theta_k)$ using the outer product $[\nabla \log \pi_{\theta_k}(a | s)  \nabla \log \pi_{\theta_k}(a |s)^\top]$ computed from samples $(s, a) \sim \nu^{\pi_{\theta_k}}$. However, estimating $\nabla J(\theta_k)$ is more challenging, as it requires access to the advantage function, which is not directly observable. We note that $A^{\pi_{\theta_k}}(s, a)$  can be expressed using $J(\theta_k)$ and $V^{\pi_{\theta_k}}$, as follows from the definition of  $A^{\pi_{\theta_k}}(s, a)$ and Bellman’s equation. Consequently, we focus next on estimating these quantities.

Notice that $\eta_k = J(\theta_k)$ is the solution to the following equation:
\begin{align} \label{eq_grad_R_eta} R_{{\theta_k}}(\eta_k) = 0, \quad \text{where} \quad R_{{\theta_k}}(\eta) = \E_{(s,a)\sim \nu^{\pi_{\theta_k}}}[\eta - r(s, a)]. \end{align}

 Given access to $R_{\theta_k}(\eta)$, the solution to the above can be computed via the following iterative update: \begin{align} \label{eq:eta-update} \eta_{h+1}^k = \eta_h^k - c_{\beta}\beta R_{\theta_k}(\eta_h^k), \end{align} where $\beta$ and $c_{\beta}$ are step-size parameters. 

% In practice, we do not know $R_{\theta}(\eta)$ exactly and instead use an estimate constructed from the observed samples as follows
% \begin{align} \label{eq:eta-update} \eta_{h+1} =\eta_h - \frac{c_{\beta}\beta}{B}\sum_{b=1}^B[\eta_h-r(s_b,a_b)], \end{align}

To approximate the value function $V^{\pi_{\theta_k}}(\cdot)$, we employ a linear critic $\hat{V}(\zeta_{\theta_k}, \cdot) = \zeta_{\theta_k}^\top \phi(\cdot)$, where $\zeta_{\theta_k} \in \mathbb{R}^m$ denotes the critic parameter and $\phi(s) \in \mathbb{R}^m$ is a feature mapping satisfying $|\phi(s)| \leq 1$ for all $s \in \mathcal{S}$. We also assume that the feature vectors ${\phi(s)}$ are linearly independent.
 %At each outer iteration $k$, we estimate $\xi_k = [\eta_k, \zeta_k]^\top$, where $\eta_k$ approximates the average reward $J(\theta_k)$ and $\zeta_k$ approximates the critic parameter $\zeta_{\theta_k}$.
The parameter $\zeta_{\theta_k}\in\mathbb{R}^m$ is obtained by solving the following optimization problem:
\begin{align}
\label{value_function_problem} \min_{\zeta \in\mathbb{R}^m} E(\theta_k, \zeta) \coloneqq \frac{1}{2}\E_{s \sim d^{\pi_{\theta_k}}} \left( V^{\pi_{\theta_k}}(s) - \hat{V}(\zeta, s) \right)^2. 
\end{align} 

This formulation enables the use of a gradient-based iterative procedure to compute $\zeta_{\theta_k}$. Specifically, the gradient of $E(\theta_k, \zeta)$ is given by:
\begin{align} 
\label{eq_nu_grad} 
\begin{split}
\nabla_{\zeta} E(\theta_k, \zeta) &= \E_{s \sim d^{\pi_{\theta_k}}} \left[\left( V^{\pi_{\theta_k}}(s) - \hat{V}(\zeta, s) \right) \nabla_{\zeta} \hat{V}(\zeta, s)\right] \\ 
&= \E_{s \sim d^{\pi_{\theta_k}}} \left[ \left( V^{\pi_{\theta_k}}(s) - \zeta^{\top}\phi(s) \right) \phi(s)\right] 
\end{split}
\end{align}

The above quantity can be approximated as follows
\begin{align}
   \nabla_{\zeta} E(\theta_k, \zeta) &= \E_{s \sim d^{\pi_{\theta_k}}} \left[ \left( V^{\pi_{\theta_k}}(s) - \zeta^{\top}\phi(s) \right) \phi(s)\right]  \\ 
&= \E_{(s,a) \sim \nu^{\pi_{\theta_k}},s'\sim P(\cdot|s,a)} \left[   \left( r(s,a)-J(\theta_k)+V^{\pi_{\theta_k}}(s') - \zeta^{\top}\phi(s) \right) \phi(s)\right]\\
&\approx \E_{(s,a) \sim \nu^{\pi_{\theta_k}},s'\sim P(\cdot|s,a)} \left[   \left( r(s,a)-J(\theta_k)+\zeta^{\top}\phi(s') - \zeta^{\top}\phi(s) \right) \phi(s)\right]\\
&\coloneqq  \widetilde{\nabla}_{\zeta}E(\theta_k, \zeta)
\end{align}
Using this gradient, $\zeta_{\theta_k}$ can be computed iteratively via: \begin{align} \label{eq:exact-zeta} \zeta_{h+1}^k = \zeta_h^k - \beta \widetilde{\nabla}_{\zeta} E(\theta_k, \zeta_h^k). \end{align}

The updates in \eqref{eq:eta-update}, and \eqref{eq:exact-zeta} can be combined into one update for $\xi_h^k = [(\eta^k_h)^\top, (\zeta_h^k)^\top]^{\top}$ as follows
\begin{align}
\label{eq:exact_xi}
    \xi_{h+1}^k = \xi_h^k - \beta [c_{\beta}R_{\theta}(\eta_h^k)^\top,\widetilde{\nabla}_{\zeta}E(\theta_k,\zeta_h^k)^\top]^\top
\end{align}

The joint update rule for $\eta^k_h$ and $\zeta^k_h$ in \eqref{eq:exact_xi} forms the basis of the critic step in Algorithm~\ref{alg:acb}, where sample averages from the trajectory are used in place of the intractable expectations.

%%===========================Algorithm==============================%%

Algorithm \ref{alg:acb} operates using a single trajectory of total length $T = 2KHB$. For each outer iteration $k$, the algorithm collects a contiguous segment of $2BH$ samples by following the current policy $\pi_{\theta_k}$, starting from the final state of the previous iteration. The first $BH$ samples from  $\pi_{\theta_k}$ are used to perform $H$ iterations of the critic update subroutine, and the next $BH$ samples are used for $H$ iterations of the NPG update.

\textbf{Critic subroutine}: For each $h \in {0, \dots, H-1}$, we use a sample averaged version of \eqref{eq:exact_xi} to obtain the update \begin{align} \label{eq:critic_reward_update} \xi_{h+1}^k = \xi_h^k - \beta \left[\frac{1}{B} \sum_{b=1}^{B} v_k(z_b^{kh}; \xi_h^k)\right], \end{align} where $\beta$ is the critic learning rate, $B$ is the batch size, $z_b^{kh} = (s_b^{kh}, a_b^{kh}, s_{b+1}^{kh})$ is a sampled transition, and $v(z_b^{kh}; \xi_h^k)$ is given by: \begin{align} \label{eq:vt_exp} v_k(z_b^{kh}; \xi_h^k) = A_{v}(\theta_k, z_b^{kh}) \xi_h^k - b_{v}(\theta_k, z_b^{kh}), \end{align} where the matrices are defined as 
\begin{equation} \label{eq_def_Az_matrix} A_{v}(\theta_k, z_b^{kh}) = \begin{bmatrix} c_\beta & 0 \\ \phi(s_b^{kh}) & \phi(s_b^{kh}) \left[\phi(s_b^{kh}) - \phi(s_{b+1}^{kh})\right]^\top \end{bmatrix}, \text{  } b_{v}(\theta_k, z_b^{kh}) = \begin{bmatrix} c_\beta r(s_b^{kh}, a_b^{kh}) \\ r(s_b^{kh}, a_b^{kh}) \phi(s_b^{kh}) \end{bmatrix} \end{equation}

The resulting update follows standard constructions in Temporal Difference learning (e.g., \cite{NEURIPS2021_096ffc29}). After $H$ critic updates, we take $\eta_k\coloneqq\eta_H^k$ and $\zeta_k\coloneqq\zeta_H^k$ as the final estimates of the average reward and critic parameters, giving $\hat{V}(\zeta_k,\cdot)=\zeta_k^\top\phi(\cdot)$ as the value estimate.

\textbf{Natural Policy Gradient estimation}: The \emph{advantage estimate} at sample $z_b^{kh}$ is computed from the Bellman equation using $\eta_k$ and $\hat{V}(\zeta_k,\cdot)$ as follows
 \begin{equation} \label{eq:adv_est} \hat{A}^{\pi_{\theta_k}}(z_b^{kh}; \xi_k) = r(s_b^{kh}, a_b^{kh}) - \eta_k + \zeta_k^\top \left[\phi(s_{b+1}^{kh}) - \phi(s_b^{kh})\right]. \end{equation}

Based on the above estimate, we construct an approximate policy gradient: 
\begin{align} \label{eq:g-est}b_{u}(\theta_k, \xi_k, z_b^{kh}) = \hat{A}^{\pi_{\theta_k}}(z_b^{kh}; \xi_k) \nabla_\theta \log \pi_{\theta_k}(a_b^{kh} | s_b^{kh}), \end{align}
and refine the natural gradient estimate $\omega_k$ via $H$ iterations of the update:
\begin{align} \label{eq:NPG_update} \omega_{h+1}^k = \omega_h^k + \gamma \left[\frac{1}{B} \sum_{b=1}^B u_k(z_b^{kh}; \xi_k)\right], 
\end{align} 
where $\gamma$ is the NPG learning rate. The update direction $u_k(z_b^{kh}; \xi_k)$ is given by \begin{align} \label{eq:ut_exp} u_k(z_b^{kh}; \xi_k) = A_{u}(\theta_k, z_b^{kh}) \omega_h^k - b_{u}(\theta_k, \xi_k, z_b^{kh}), \end{align} with \begin{align} A_{u}(\theta_k, z_b^{kh}) &= \nabla_\theta \log \pi_{\theta_k}(a_b^{kh} | s_b^{kh})  \nabla_\theta \log \pi_{\theta_k}(a_b^{kh}| s_b^{kh})^\top. \label{eq:ahatu}\end{align}

Set $\omega_k\coloneqq\omega_H^k$. The policy is updated as: $ \label{eq:policy_update} \theta_{k+1} = \theta_k + \alpha \omega_k,$ where $\alpha$ is the policy learning rate.

 %This design ensures that the algorithm runs on a single, uninterrupted trajectory while reusing samples efficiently for both actor and critic updates.

%Based on \eqref{eq_grad_R_eta} and \eqref{eq_nu_grad}, the term $v(s_b^{kh},a_b^{kh};\xi^k_h)$ can be thought of as a crude approximation of $v(s_b^{kh},a_b^{kh};\xi^k_h)$ obtained using a single transition, $z_t^{kh}$.

\section{Regret Guarantee for Unichain MDP}

%We first state some relevant assumptions. Before proceeding, we introduce some notations. Let $A_k \coloneqq \E_{\theta_k}\left[A_k(z)\right]$, and $b_k \coloneqq \E_{\theta_k} \left[b_k(z)\right]$ where $A_k(z)$, $b_k(z)$ are described in \eqref{eq_def_Az_matrix}, \eqref{eq_def_Az_matrix} and the expectation $\E_{\theta_k}$ is computed over the distribution of $z=(s, a, s')$ where $(s, a)\sim \nu^{\pi_{\theta_k}}$, $s'\sim P(\cdot|s, a)$. 
\subsection{Assumptions}
\label{sec:assump}
%Discussions TBA: Markov chain properties and assumption significance - much weaker which includes periodic and transient states, satisfied when unique recurrent class (unique dist from eigenvalue analysis and invoking rosenthal lemma), critic discussion using van roy analysis. 

Let $A_{v}(\theta) \coloneqq \E_{\theta}\left[ A_{v}(\theta, z)\right]$, and $b_{v}(\theta) \coloneqq \E_{\theta} \left[b_{v}(\theta, z)\right]$ where $A_{v}(\theta, z)$, $b_{v}(\theta, z)$ are defined in \eqref{eq_def_Az_matrix} and \eqref{eq:g-est} and the expectation $\E_\theta$ is computed over the distribution of $z=(s, a, s')$ where $(s, a)\sim \nu^{\pi_\theta}$, $s'\sim P(\cdot|s, a)$. For arbitrary policy parameter $\theta$, we denote $\xi_{\theta}^*=[A_{v}(\theta)]^{\dagger}b_{v}(\theta)=[\eta_{\theta}^*, \zeta_{\theta}^*]^{\top}$. Using these notations, below we state some assumptions related to the critic analysis.
\begin{definition}
\label{def:critic-error}
    We define the critic approximation error, $\epsilon_{\mathrm{app}}$ as follows.
\begin{align}
    \epsilon_{\mathrm{app}} = \sup_{\theta} \inf_{\zeta} \left\{\frac{1}{2}\E_{s \sim d^{\pi_{\theta}}} \left( V^{\pi_{\theta}}(s) - \hat{V}(\zeta, s) \right)^2\right\}. 
\end{align}

\end{definition}

The definition~\ref{def:critic-error} is widely adopted in the literature \citep{suttle2023beyond,ganesh2024order,xu2019sample,patel2024global,chen2023finitetime}, although there may be minor variations in notation, such as the inclusion of a square root. This definition is closely linked to the representational capacity of the chosen feature map, with $\epsilon_{\mathrm{app}}$ characterizing the resulting approximation error. A well-constructed feature map can yield a small or even vanishing $\epsilon_{\mathrm{app}}$.

\begin{assumption} \label{assum:critic_positive_definite}
Let $M_{\theta} \coloneqq \mathbb{E}_{\theta} \left[ \phi(s) \left( \phi(s) - \phi(s') \right)^\top \right]$,  
where the expectation $\mathbb{E}_{\theta}$ is taken over $s \sim d^{\pi_{\theta}}$ and $s' \sim P^{\pi_{\theta}}(s, \cdot)$. Then, for all $\theta$, there exists a constant $\lambda > 0$ such that for all $x \in \ker(M_{\theta})^{\perp}$ (i.e., the subspace orthogonal to the kernel of $M_{\theta}$), the following inequality holds:
\begin{align}
    x^\top M_{\theta} x \geq \lambda \| x \|^2.
\end{align}
\end{assumption}
 
% Both Assumptions \ref{assump:critic-error} and \ref{assum:critic_positive_definite} are frequently employed in the analysis of actor-critic methods \citep{suttle2023beyond,patel2024global,wu2020finite,panda2024two}. 

\textbf{Comments on Assumption \ref{assum:critic_positive_definite}:}  We emphasize that this condition is substantially weaker than the commonly imposed requirement of strict positive-definiteness of $M_{\theta}$, which appears in nearly all Actor-Critic works \citep{ganesh2024order,patel2024global,wang2024nonasymptotic,panda2024two,suttle2023beyond}. In Section~\ref{subsec:main-result}, we explain why the strict positive definiteness of the critic matrix cannot be guaranteed in the unichain case, as opposed to the ergodic case, and discuss the resulting analytical challenges.

%{\bf Swetha, first - for $\lambda>0$, thus us still assuming positive definite. Second, I am not sure what you mean by "our own analysis of the critic still encounters  challenges due to non-ergodicity" - are you saying that this Zhang et al do not work in our case? I do not see that - need to say why that do not work? and what will be novelty compared to that?}

We will now state some assumptions related to the policy parameterization. Before proceeding, we define the policy approximation error.
\begin{definition}
\label{assump:function_approx_error}
Define $\epsilon_{\mathrm{bias}}$ as the least upper bound on the \textit{transferred compatible function approximation error}, $L_{\nu^{\pi^*}}(\omega_{\theta}^*; \theta)$, i.e.,
\begin{align*}
\epsilon_{\mathrm{bias}} \coloneqq \sup_{\theta} L_{\nu^{\pi^*}}(\omega_{\theta}^*; \theta) 
\end{align*} 
where $\omega_{\theta}^*$ denotes the exact Natural Policy Gradient direction at $\theta$ \eqref{eq_exact_npg}, $\pi^*$ indicates the optimal policy, and $L_\nu$ is defined in \eqref{eq_def_L_nu}.
\end{definition}
The term $\epsilon_{\mathrm{bias}}$ is a standard quantity in the literature on parameterized PG methods \citep{agarwal2020optimality,fatkhullin2023stochastic,mondal2024improved,Masiha_SCRN_KL,suttle2023beyond,wang2024nonasymptotic,ganesh2024order,ganesh2024variance,xu2019sample}, and it reflects the expressivity of the chosen policy class. For example, when the policy class is expressive enough to represent any stochastic policy, such as with softmax parameterization, we have $\epsilon_{\mathrm{bias}} = 0$ \citep{agarwal2021theory}. In contrast, under more restrictive parameterizations that do not cover all stochastic policies, we may have $\epsilon_{\mathrm{bias}} > 0$. Nevertheless, this bias is often negligible when using rich neural network parameterizations \citep{wang2019neural}.

\begin{assumption}
    \label{assump:score_func_bounds}
    For all $\theta, \theta_1,\theta_2$ and $(s,a)\in\mathcal{S}\times\mathcal{A}$, the following statements hold.
    \begin{align}
   (a) \text{  }\Vert\nabla_{\theta}\log\pi_\theta(a\vert s)\Vert\leq G_1 \quad (b) \text{  } \Vert \nabla_{\theta}\log\pi_{\theta_1}(a\vert s)-\nabla_\theta\log\pi_{\theta_2}(a\vert s)\Vert\leq G_2\Vert \theta_1-\theta_2\Vert \nonumber
    \end{align}
\end{assumption}

\begin{assumption}[Fisher non-degenerate policy]
    \label{assump:FND_policy}
    There exists a constant $\mu>0$ such that $F(\theta)-\mu I_{d}$ is positive semidefinite where $I_{d}$ denotes an identity matrix.
\end{assumption}

{\bf Comments on Assumptions \ref{assump:function_approx_error}-\ref{assump:FND_policy}:}  We emphasize that these are standard in the PG literature \citep{liu2020improved,fatkhullin2023stochastic,mondal2024improved,Masiha_SCRN_KL,suttle2023beyond,wang2024nonasymptotic,ganesh2024order,ganesh2024variance}. Assumption~\ref{assump:score_func_bounds} stipulates that the score function is both bounded and Lipschitz-continuous, an assumption frequently used in error decomposition and in bounding the policy gradient. Assumption~\ref{assump:FND_policy} requires that the eigenvalues of the Fisher information matrix are uniformly bounded away from zero, a condition commonly employed to establish global complexity guarantees for PG methods. Notably, Assumptions~\ref{assump:score_func_bounds}-\ref{assump:FND_policy} have recently been verified for a range of policy classes, including Gaussian and Cauchy distributions having parameterized means with clipping \citep{liu2020improved,fatkhullin2023stochastic}.

\subsection{Main Result}
\label{subsec:main-result}
The following theorem gives the regret bound for the proposed algorithm. 

\begin{theorem}[Main Result]
\label{thm:main}
Consider Algorithm~\ref{alg:acb} and suppose Assumptions~\ref{assump_mdp}–\ref{assump:FND_policy} hold. Let $J$ be $L$ smooth and set $K=\Theta(\sqrt{T}/(\log T))$, $B=\Theta(\sqrt{T})$ and $H=\Theta(\log T)$. Then, for a suitable choice of learning parameters, the expected regret satisfies
\begin{align}
    \E[\mathrm{Reg}_T] \leq \Tilde{\mathcal{O}} \left( T(\sqrt{\epsilon_{\mathrm{bias}}} + \sqrt{\epsilon_{\mathrm{app}}}) + \sqrt{T}(C^2C_{\mathrm{tar}}+C_{\mathrm{tar}}) \right),
\end{align}
where $C\coloneqq C_{\mathrm{tar}}+C_{\mathrm{hit}}$.
\end{theorem}

The choice of learning rates along with the regret bound containing all problem-specific constants is provided in Appendix~\ref{app:final}. We now provide a brief overview of the proof. We first begin with a regret decomposition obtained using standard techniques:
\begin{lemma}
\label{lem:regret-decomp}
Consider the setting in Theorem \ref{thm:main}. Then, the expected regret of Algorithm~\ref{alg:acb} satisfies
\begin{align}\label{eq:general_bound}
			&\E [\mathrm{Reg}_T]\leq T\sqrt{\epsilon_{\mathrm{bias}}}+HBG_1\sum_{k=0}^{K-1}\E\Vert(\E_k\left[\omega_k\right]-\omega^*_k)\Vert +\alpha G_2HB\sum_{k=0}^{K-1}\E\Vert \omega_k-\omega_k^*\Vert^2\nonumber \\
            &+\frac{\alpha G_2HB}{\mu^2}\sum_{k=0}^{K-1}\E\Vert \nabla_{\theta} J(\theta_k)\Vert^2+\frac{HB}{\alpha }\E_{s\sim d^{\pi^*}}[\mathrm{KL}(\pi^*(\cdot\vert s)\Vert\pi_{\theta_0}(\cdot\vert s))]+2C(K+1).
		\end{align}
\end{lemma}

The proof of the above is given in Appendix~\ref{app:proof-lemma1}. This regret decomposition crucially depends on the bias, $\E\Vert(\E_k\left[\omega_k\right]-\omega^*_k)\Vert$, and second-order error of the NPG estimate, $\E\Vert \omega_k-\omega_k^*\Vert^2$. These in turn dependent on the quality of the critic estimates. Hence the main focus of our proof consists of bounding these terms. The analysis of remaining terms then follows from these bounds; details are provided in Appendix~\ref{app:final}. 

\textbf{Properties of Markovian Sample Average}: Analyses in the ergodic setting often rely on the property of exponentially fast mixing. In contrast, even irreducible but periodic Markov chains do not exhibit analogous exponential mixing behavior. A common strategy to address periodicity is to apply time-averaging, for which convergence results are available in the literature for irreducible, periodic chains \citep{Roberts01011997}. In what follows, we extend such results to Markov chains with a single recurrent class.

\begin{lemma} 
\label{lem:ergodic-avg}
    Let Assumption \ref{assump_mdp} hold and consider any $\theta \in \Pi_\Theta$. Then, the following bound holds:  
\begin{equation}
\label{eq:avg2}
    \norm{\frac{1}{t} \sum_{i=1}^{t} (P^{\pi_{\theta}})^i(s_0, \cdot) - d^{\pi_{\theta}}(\cdot)}_{\TV} \leq \frac{C}{t}, \quad \forall s_0 \in \mathcal{S}.
\end{equation} 
\end{lemma}

Using the result above, we show in the following lemma that employing a batch size of $B$ reduces both the bias and variance of the estimate by a factor of $1/B$. Accordingly, in Algorithm~\ref{alg:acb}, we choose a large batch size $B = \Theta(\sqrt{T})$ to obtain an order-optimal dependence on $T$ in the regret bound.

\begin{lemma}
\label{lem:bias-variance-unichain}
    Let Assumption~\ref{assump_mdp} hold, and let $f : \mathcal{S} \to \mathbb{R}^d$ satisfy $\|f(s)\| \leq C_f$ for all $s \in \mathcal{S}$, for some constant $C_f > 0$. Then, the following bounds hold:
    \begin{align*}
    (i)\text{   } \norm{\E\left[\frac{1}{B}\sum_{i=1}^Bf(s_i)\right]-\mu}  \leq \frac{\sqrt{d}C_f C}{B} \quad (ii)\text{   } \E \norm{\left[\frac{1}{B}\sum_{i=1}^Bf(s_i)\right]-\mu}^2 \leq \frac{C_f^2 + 2\sqrt{d}C_f^2C }{B}
\end{align*}
where $\E$ denotes the expectation of the Markov chain $\{s_i\}$ induced by $\pi_{\theta}$ starting from any $s_0\in \mathcal{S}$ and $\mu = \E_{s \sim d^{\pi_{\theta}}}[f(s)]$.
\end{lemma}

Another consequence of Lemma~\ref{lem:ergodic-avg} is in providing bounds for the value and $Q$ functions. In the ergodic case, it is known that the functions $V^{\pi}$ and $Q^{\pi}$ are $\cO(t_{\mathrm{mix}})$. In the unichain case, using the above result, we show that these functions are still bounded, by $\cO(C)$ instead. The proof of these results are given in Appendix~\ref{app:unichain}. We next describe the error analysis of the critic and NPG subroutines in Algorithm \ref{alg:acb}. 

% \begin{lemma}[Variance of Markovian Sample Average]
% \label{lem:variance-markov}
%     Let Assumption~\ref{assump_mdp} hold and consider a Markov chain $\{s_n\}_{n\geq 1}\subset \mathcal{S}$ generated from $\pi_{\theta}$ starting from a distribution $p$. Let $f:\mathcal{S}\to \R^d$ be a uniformly bounded function on $\mathcal{S}$, i.e., there exists $C_f>0$ such that $\|f(s)\|\leq C_f,\forall s\in \mathcal{S}$. Then, the following bound holds:
%     \begin{align}
%     \E_{s_1\sim p}\norm{\left[\frac{1}{B}\sum_{i=1}^Bf(s_i)\right]-\mu}^2 \leq \frac{C_f^2 + 2\sqrt{d}C_f^2C_{\mathrm{tar}} }{B}.
% \end{align}
% where $\mu = \E_{s \sim d^{\pi_{\theta}}}[f(s)]$.
% \end{lemma}

\textbf{Challenges due to transient states}: We observe that the critic and NPG updates can be interpreted as stochastic linear updates with Markovian noise. From the preceding lemmas, it follows that the bias and variance of the noise in these averaged update directions can be made sufficiently small. Nevertheless, the analysis still presents certain challenges.

We show that the kernel of the critic matrix $A_v(\theta)$ can change depending on the policy. More specifically, the kernel will depend on the set of transient states (see Lemma \ref{lem:critic_psd}). This is unlike in the ergodic case, where the kernel remains the same for all policies. As a result, strict positive definiteness of the critic matrix cannot, in general, be ensured by using a carefully selected set of feature vectors. 

Although a prior work in the ergodic case has considered the case where the critic matrix is singular \citep{NEURIPS2021_096ffc29}, our approach differs significantly, in addition to incorporating the sample averaging analysis discussed in the previous section. In particular, our regret guarantees require a sharp bias bound for the critic updates, $\|\mathbb{E}[\xi_h] - \xi^*\|$, in addition to the standard second-order error bound $\mathbb{E}[\|\xi_h - \xi^*\|^2]$. The analysis in \citep{NEURIPS2021_096ffc29} only addresses the second-order error.
Furthermore, to establish our convergence guarantees, we require the following condition:
\begin{align}
\label{eq:cond_critic}
\ker(A_v(\theta)) \subseteq \ker(A_v(z,\theta)).
\end{align}
Equation \eqref{eq:cond_critic} holds in the ergodic case, even if the critic matrix is singular, but generally does not hold in unichain settings (see Appendix~\ref{app:critic}). However, it is satisfied when the Markov chain is restricted to its recurrent class. To address this, we show that, with high probability, the chain induced by $\pi_{\theta_k}$ enters the recurrent class, $\mathcal{S}^{\theta_k}_R$, very early, allowing us to focus on the recurrent class.

More specifically, recall that in each outer iteration $k$, Algorithm~\ref{alg:acb} collects a contiguous segment of $2BH$ samples under $\pi_{\theta_k}$. We define the event $\mathcal{E}_B^k := \{ T_{\theta_k} \leq B \}$, where $T_{\theta_k}$ is the first hitting time of $\mathcal{S}^{\theta_k}_R$ for this segment. Lemma~\ref{lem:prob-exponential} shows that $\Pr((\mathcal{E}_B^k)^c) \leq 2^{-\lfloor B / 2C_{\mathrm{hit}} \rfloor}$, which implies that the chain reaches the recurrent class within the first $B$ steps with overwhelmingly high probability. We perform our analysis conditioned on this event.

Let $\mathcal{E}_B \coloneqq \bigcap_{k=1}^K \mathcal{E}_B^k$. We decompose the expected regret based on whether $\mathcal{E}_B$ holds:
$$
\mathbb{E}[\mathrm{Reg}_T] = \mathbb{E}[\mathrm{Reg}_T | \mathcal{E}_B] \Pr(\mathcal{E}_B) + \mathbb{E}[\mathrm{Reg}_T | \mathcal{E}_B^c] \Pr(\mathcal{E}_B^c).
$$
Since the instantaneous regret is at most 1, the total regret is bounded above by $T$, and applying a union bound over $K$ yields
$$
\mathbb{E}[\mathrm{Reg}_T] \leq \mathbb{E}[\mathrm{Reg}_T | \mathcal{E}_B] + T K 2^{-\lfloor B / 2C_{\mathrm{hit}} \rfloor}.
$$
This decomposition allows us to isolate the high-probability event $\mathcal{E}_B$, under which all policies quickly enter their recurrent classes. By choosing $B = \sqrt{T}$, the second term becomes negligible for large $T$, and we may therefore restrict our analysis to the case where $\mathcal{E}_B$ holds.

\textbf{Analysis under event $\mathcal{E}_B$}: By standard Markov chain theory, $T_{\theta}$ is a stopping time hence, the Strong Markov Property (Theorem 1.2.5 of \citep{norris1998markov}) applies at \( T_{\theta} \). In particular, conditioned on \( s_T \), the post-\( T_{\theta} \) trajectory is independent of the pre-\( T_{\theta} \) path and evolves as a Markov chain initialized at \( s_T \) under the original kernel \( P^{\pi_{\theta}} \). As a result, it suffices to analyze the critic and NPG subroutines with the samples restricted to the irreducible class. In Appendix~\ref{app:general_linear}, we analyze a generic stochastic linear recursion and show that these subroutines can be viewed as special cases of this recursion.

 \begin{theorem}
\label{thm:critic-final}
Consider Algorithm~\ref{alg:acb} under the assumptions of Theorem~\ref{thm:main}. Conditioned on the event $\mathcal{E}_B$, the critic estimate $\xi_k$ satisfies
\begin{align*}
\textstyle{\E\left[\|\Pi(\xi_k - \xi^*_k)\|^2|\mathcal{E}_B\right] \leq \Tilde{\cO}\Bigg(e^{-c_1H}\|\xi_0-\xi^*\|^2 +  \frac{C_{\mathrm{tar}}\sqrt{m}}{\lambda^{4}B}  \Bigg)}
\end{align*}
and
\begin{align*}
\|\Pi(\E[\xi_k|\mathcal{E}_B] - \xi^*_k)\|^2 \leq \textstyle{\Tilde{\cO}\Bigg(e^{-c_1H} \|\xi_0-\xi^*\|^2 +\frac{C_{\mathrm{tar}}^2 m}{\lambda^6 B^2}\Bigg)}
\end{align*}
where $c_1=\tfrac{\lambda^3}{16}$, $\xi_{k}^*=[A_{v}(\theta_k)]^{\dagger}b_{v}(\theta_k)=[\eta_{k}^*, \zeta_{k}^*]^{\top}$ and $\Pi$ denotes the projection onto the space $\ker(A_{v}(\theta))^{\perp}$.
\end{theorem}

The details regarding the verification of all conditions of the generic linear recursion used in Theorem~\ref{thm_2} for the critic update is provided in Appendix~\ref{app:critic}.

\begin{theorem}
\label{thm:NPG-final}
Consider Algorithm~\ref{alg:acb} under the assumptions of Theorem~\ref{thm:main}. Conditioned on the event $\mathcal{E}_B$, the NPG estimate $\omega_k$ satisfies
\begin{align*}
&\textstyle{\E\left[\|\omega_k - \omega_k^*\|^2|\mathcal{E}_B\right] \leq \cO\Bigg(e^{-c_2H}\|\omega_0-\omega_k^*\|^2 + \frac{\sqrt{d}G_1^4C_{\mathrm{tar}}C^2}{\mu^2 B}+\frac{G_1^2\E[\|\Pi(\xi_k-\xi_k^*)\|^2|\mathcal{E}_B]}{\mu^{2}}+\frac{G_1^2 \epsilon_{\mathrm{app}}}{\mu^{2}}\Bigg)}
\end{align*}
and
\begin{align*}
&\textstyle{\|\E[\omega_k|\mathcal{E}_B] - \omega_k^*\|^2 \leq \cO\Bigg(e^{-c_2H}\|\omega_0-\omega_k^*\|^2+\frac{dG_1^8 C_{\mathrm{tar}}^2}{\mu^6 B^2}+\frac{G_1^2\|\Pi(\E[\xi_k|\mathcal{E}_B]-\xi_k^*)\|^2}{\mu^2}+\frac{G_1^2 \epsilon_{\mathrm{app}}}{\mu^2}\Bigg)}
\end{align*}
where $c_2=\tfrac{\mu^2}{4(1+4C)G_1}$.
\end{theorem}

Similar to the critic update, the details regarding the verification of all conditions of the generic linear recursion used in Theorem~\ref{thm_2} for the NPG update is provided in Appendix~\ref{app:npg-final}.

\section{Conclusion}

We have presented NAC-B, a Natural Actor-Critic algorithm with batching, that achieves $\tilde{O}(\sqrt{T})$ regret in infinite-horizon average-reward MDPs under the unichain assumption. Unlike prior work that relies on stronger ergodicity conditions, NAC-B operates under a weaker structural assumption that allows for both transient states and periodic behavior, thereby broadening the applicability of policy gradient methods in average-reward settings. Our approach leverages function approximation for scalability and introduces batching to mitigate issues arising from periodicity and high variance in gradient estimates. Future works include extending this analysis to other settings such as constrained MDPs \cite{bai2024learning,xu2025global}, neural critics \cite{ganesh2025orderoptimal} and relaxing the unichain assumption, though we anticipate this to be a challenging task (See Remark \ref{rem:ext}).

\bibliography{references}

\begin{thebibliography}{46}
\providecommand{\natexlab}[1]{#1}
\providecommand{\url}[1]{\texttt{#1}}
\expandafter\ifx\csname urlstyle\endcsname\relax
  \providecommand{\doi}[1]{doi: #1}\else
  \providecommand{\doi}{doi: \begingroup \urlstyle{rm}\Url}\fi

\bibitem[Agarwal et~al.(2020)Agarwal, Kakade, Lee, and Mahajan]{agarwal2020optimality}
A.~Agarwal, S.~M. Kakade, J.~D. Lee, and G.~Mahajan.
\newblock Optimality and approximation with policy gradient methods in markov decision processes.
\newblock In \emph{Conference on Learning Theory}, pages 64--66, 2020.

\bibitem[Agarwal et~al.(2021)Agarwal, Kakade, Lee, and Mahajan]{agarwal2021theory}
A.~Agarwal, S.~M. Kakade, J.~D. Lee, and G.~Mahajan.
\newblock On the theory of policy gradient methods: Optimality, approximation, and distribution shift.
\newblock \emph{The Journal of Machine Learning Research}, 22\penalty0 (1):\penalty0 4431--4506, 2021.

\bibitem[Agrawal and Agrawal(2024)]{agrawal2025optimistic}
P.~Agrawal and S.~Agrawal.
\newblock Optimistic q-learning for average reward and episodic reinforcement learning.
\newblock \emph{arXiv preprint arXiv:2407.13743}, 2024.

\bibitem[Agrawal and Jia(2017)]{agrawal2017optimistic}
S.~Agrawal and R.~Jia.
\newblock Optimistic posterior sampling for reinforcement learning: worst-case regret bounds.
\newblock \emph{Advances in Neural Information Processing Systems}, 30, 2017.

\bibitem[Al-Abbasi et~al.(2019)Al-Abbasi, Ghosh, and Aggarwal]{al2019deeppool}
A.~O. Al-Abbasi, A.~Ghosh, and V.~Aggarwal.
\newblock Deeppool: Distributed model-free algorithm for ride-sharing using deep reinforcement learning.
\newblock \emph{IEEE Transactions on Intelligent Transportation Systems}, 20\penalty0 (12):\penalty0 4714--4727, 2019.

\bibitem[Aldous and Fill(2002)]{aldous2002reversible}
D.~Aldous and J.~A. Fill.
\newblock \emph{Reversible Markov Chains and Random Walks on Graphs}.
\newblock 2002.
\newblock URL \url{https://www.stat.berkeley.edu/~aldous/RWG/book.html}.

\bibitem[Auer et~al.(2008)Auer, Jaksch, and Ortner]{auer2008near}
P.~Auer, T.~Jaksch, and R.~Ortner.
\newblock Near-optimal regret bounds for reinforcement learning.
\newblock In D.~Koller, D.~Schuurmans, Y.~Bengio, and L.~Bottou, editors, \emph{Advances in Neural Information Processing Systems}, volume~21, 2008.

\bibitem[Bai et~al.(2024{\natexlab{a}})Bai, Mondal, and Aggarwal]{bai2024learning}
Q.~Bai, W.~Mondal, and V.~Aggarwal.
\newblock Learning general parameterized policies for infinite horizon average reward constrained mdps via primal-dual policy gradient algorithm.
\newblock \emph{Advances in Neural Information Processing Systems}, 37:\penalty0 108566--108599, 2024{\natexlab{a}}.

\bibitem[Bai et~al.(2024{\natexlab{b}})Bai, Mondal, and Aggarwal]{bai2023regret}
Q.~Bai, W.~U. Mondal, and V.~Aggarwal.
\newblock Regret analysis of policy gradient algorithm for infinite horizon average reward markov decision processes.
\newblock In \emph{AAAI Conference on Artificial Intelligence}, 2024{\natexlab{b}}.

\bibitem[Bhatia(1997)]{bhatia1997matrix}
R.~Bhatia.
\newblock \emph{Matrix Analysis}, volume 169 of \emph{Graduate Texts in Mathematics}.
\newblock Springer-Verlag, New York, 1997.

\bibitem[Chen et~al.(2025)Chen, Zhou, Chen, Pedramfar, Lan, Zhu, Zhou, Mauri~Ruiz, Kumar, Dong, and Aggarwal]{chen2024learning}
C.-L. Chen, H.~Zhou, J.~Chen, M.~Pedramfar, T.~Lan, Z.~Zhu, C.~Zhou, P.~Mauri~Ruiz, N.~Kumar, H.~Dong, and V.~Aggarwal.
\newblock Learning-based two-tiered online optimization of region-wide datacenter resource allocation.
\newblock \emph{IEEE Transactions on Network and Service Management}, 22\penalty0 (1):\penalty0 572--581, 2025.

\bibitem[Chen and Zhao(2023)]{chen2023finitetime}
X.~Chen and L.~Zhao.
\newblock Finite-time analysis of single-timescale actor-critic.
\newblock In \emph{Thirty-seventh Conference on Neural Information Processing Systems}, 2023.

\bibitem[Ding et~al.(2022)Ding, Zhang, and Lavaei]{ding-et-al22}
Y.~Ding, J.~Zhang, and J.~Lavaei.
\newblock On the global optimum convergence of momentum-based policy gradient.
\newblock In \emph{International Conference on Artificial Intelligence and Statistics}, pages 1910--1934, 2022.

\bibitem[Fatkhullin et~al.(2023)Fatkhullin, Barakat, Kireeva, and He]{fatkhullin2023stochastic}
I.~Fatkhullin, A.~Barakat, A.~Kireeva, and N.~He.
\newblock Stochastic policy gradient methods: Improved sample complexity for fisher-non-degenerate policies.
\newblock In \emph{International Conference on Machine Learning}, pages 9827--9869, 2023.

\bibitem[Ganesh et~al.(2025{\natexlab{a}})Ganesh, Chen, Mondal, and Aggarwal]{ganesh2025orderoptimal}
S.~Ganesh, J.~Chen, W.~U. Mondal, and V.~Aggarwal.
\newblock Order-optimal global convergence for actor-critic with general policy and neural critic parametrization.
\newblock In \emph{The 41st Conference on Uncertainty in Artificial Intelligence}, 2025{\natexlab{a}}.

\bibitem[Ganesh et~al.(2025{\natexlab{b}})Ganesh, Mondal, and Aggarwal]{ganesh2024order}
S.~Ganesh, W.~U. Mondal, and V.~Aggarwal.
\newblock Order-optimal global convergence for average reward reinforcement learning via actor-critic approach.
\newblock In \emph{The 42nd International Conference on Machine Learning (ICML)}, 2025{\natexlab{b}}.

\bibitem[Ganesh et~al.(2025{\natexlab{c}})Ganesh, Mondal, and Aggarwal]{ganesh2024variance}
S.~Ganesh, W.~U. Mondal, and V.~Aggarwal.
\newblock Order-optimal regret with novel policy gradient approaches in infinite horizon average reward {MDP}s.
\newblock In \emph{The 28th International Conference on Artificial Intelligence and Statistics}, 2025{\natexlab{c}}.

\bibitem[Geng et~al.(2020)Geng, Lan, Aggarwal, Yang, and Xu]{geng2020multi}
N.~Geng, T.~Lan, V.~Aggarwal, Y.~Yang, and M.~Xu.
\newblock A multi-agent reinforcement learning perspective on distributed traffic engineering.
\newblock In \emph{2020 IEEE 28th International Conference on Network Protocols (ICNP)}, pages 1--11. IEEE, 2020.

\bibitem[Golub and Van~Loan(2013)]{golub2013matrix}
G.~H. Golub and C.~F. Van~Loan.
\newblock \emph{Matrix Computations}.
\newblock Johns Hopkins University Press, Baltimore, MD, 4 edition, 2013.

\bibitem[Hong and Tewari(2025)]{hong2025comp}
K.~Hong and A.~Tewari.
\newblock A computationally efficient algorithm for infinite-horizon average-reward linear {MDP}s.
\newblock In \emph{Forty-second International Conference on Machine Learning}, 2025.

\bibitem[Jaksch et~al.(2010)Jaksch, Ortner, and Auer]{jaksch2010near}
T.~Jaksch, R.~Ortner, and P.~Auer.
\newblock Near-optimal regret bounds for reinforcement learning.
\newblock \emph{The Journal of Machine Learning Research}, 11:\penalty0 1563--1600, 2010.

\bibitem[Levin and Peres(2017)]{levin2017markov}
D.~A. Levin and Y.~Peres.
\newblock \emph{Markov chains and mixing times}, volume 107.
\newblock American Mathematical Soc., 2017.

\bibitem[Li et~al.(2024)Li, Wu, and Lan]{li2024stochastic}
T.~Li, F.~Wu, and G.~Lan.
\newblock Stochastic first-order methods for average-reward markov decision processes.
\newblock \emph{Mathematics of Operations Research}, 0\penalty0 (0), 2024.
\newblock \doi{10.1287/moor.2023.1427}.
\newblock To appear.

\bibitem[Liu et~al.(2020)Liu, Zhang, Basar, and Yin]{liu2020improved}
Y.~Liu, K.~Zhang, T.~Basar, and W.~Yin.
\newblock An improved analysis of (variance-reduced) policy gradient and natural policy gradient methods.
\newblock \emph{Advances in Neural Information Processing Systems}, 33:\penalty0 7624--7636, 2020.

\bibitem[Masiha et~al.(2022)Masiha, Salehkaleybar, He, Kiyavash, and Thiran]{Masiha_SCRN_KL}
S.~Masiha, S.~Salehkaleybar, N.~He, N.~Kiyavash, and P.~Thiran.
\newblock Stochastic second-order methods provably beat sgd for gradient-dominated functions.
\newblock \emph{arXiv preprint arXiv:2205.12856v1}, 2022.

\bibitem[Mondal and Aggarwal(2024)]{mondal2024improved}
W.~U. Mondal and V.~Aggarwal.
\newblock Improved sample complexity analysis of natural policy gradient algorithm with general parameterization for infinite horizon discounted reward markov decision processes.
\newblock In \emph{International Conference on Artificial Intelligence and Statistics}, pages 3097--3105, 2024.

\bibitem[Norris(1998)]{norris1998markov}
J.~Norris.
\newblock \emph{Markov Chains}.
\newblock Cambridge University Press, 1998.

\bibitem[Panda and Bhatnagar(2025)]{panda2024two}
P.~Panda and S.~Bhatnagar.
\newblock Two-timescale critic-actor for average reward mdps with function approximation.
\newblock In \emph{Proceedings of the 39th AAAI Conference on Artificial Intelligence (AAAI)}, 2025.

\bibitem[Patel et~al.(2024)Patel, Suttle, Koppel, Aggarwal, Sadler, Bedi, and Manocha]{patel2024global}
B.~Patel, W.~A. Suttle, A.~Koppel, V.~Aggarwal, B.~M. Sadler, A.~S. Bedi, and D.~Manocha.
\newblock Global optimality without mixing time oracles in average-reward rl via multi-level actor-critic.
\newblock In \emph{International Conference on Machine Learning}, 2024.

\bibitem[Puterman(2014)]{puterman2014markov}
M.~L. Puterman.
\newblock \emph{Markov decision processes: discrete stochastic dynamic programming}.
\newblock John Wiley \& Sons, 2014.

\bibitem[Roberts and Rosenthal(1997)]{Roberts01011997}
G.~O. Roberts and J.~S. Rosenthal.
\newblock Shift-coupling and convergence rates of ergodic averages.
\newblock \emph{Communications in Statistics. Stochastic Models}, 13\penalty0 (1):\penalty0 147--165, 1997.
\newblock \doi{10.1080/15326349708807418}.

\bibitem[Suttle et~al.(2023)Suttle, Bedi, Patel, Sadler, Koppel, and Manocha]{suttle2023beyond}
W.~A. Suttle, A.~Bedi, B.~Patel, B.~M. Sadler, A.~Koppel, and D.~Manocha.
\newblock Beyond exponentially fast mixing in average-reward reinforcement learning via multi-level monte carlo actor-critic.
\newblock In \emph{International Conference on Machine Learning}, pages 33240--33267, 2023.

\bibitem[Sutton et~al.(1999)Sutton, McAllester, Singh, and Mansour]{sutton1999policy}
R.~S. Sutton, D.~McAllester, S.~Singh, and Y.~Mansour.
\newblock Policy gradient methods for reinforcement learning with function approximation.
\newblock \emph{Advances in neural information processing systems}, 12, 1999.

\bibitem[Tsitsiklis and Van~Roy(1997)]{roy1997}
J.~Tsitsiklis and B.~Van~Roy.
\newblock An analysis of temporal-difference learning with function approximation.
\newblock \emph{IEEE Transactions on Automatic Control}, 42\penalty0 (5):\penalty0 674--690, 1997.
\newblock \doi{10.1109/9.580874}.

\bibitem[Wang et~al.(2022)Wang, Wang, and Yang]{wang2022near}
J.~Wang, M.~Wang, and L.~F. Yang.
\newblock Near sample-optimal reduction-based policy learning for average reward mdp.
\newblock \emph{arXiv preprint arXiv:2212.00603}, 2022.

\bibitem[Wang et~al.(2019)Wang, Cai, Yang, and Wang]{wang2019neural}
L.~Wang, Q.~Cai, Z.~Yang, and Z.~Wang.
\newblock Neural policy gradient methods: Global optimality and rates of convergence.
\newblock In \emph{International Conference on Learning Representations}, 2019.

\bibitem[Wang et~al.(2024)Wang, Wang, Zhou, and Zou]{wang2024nonasymptotic}
Y.~Wang, Y.~Wang, Y.~Zhou, and S.~Zou.
\newblock Non-asymptotic analysis for single-loop (natural) actor-critic with compatible function approximation.
\newblock In \emph{International Conference on Machine Learning}, 2024.

\bibitem[Wei et~al.(2020)Wei, Jahromi, Luo, Sharma, and Jain]{wei2020model}
C.-Y. Wei, M.~J. Jahromi, H.~Luo, H.~Sharma, and R.~Jain.
\newblock Model-free reinforcement learning in infinite-horizon average-reward markov decision processes.
\newblock In \emph{International conference on machine learning}, pages 10170--10180, 2020.

\bibitem[Wei et~al.(2021)Wei, Jahromi, Luo, and Jain]{wei2021learning}
C.-Y. Wei, M.~J. Jahromi, H.~Luo, and R.~Jain.
\newblock Learning infinite-horizon average-reward mdps with linear function approximation.
\newblock In \emph{International Conference on Artificial Intelligence and Statistics}, pages 3007--3015, 2021.

\bibitem[Xu et~al.(2019)Xu, Gao, and Gu]{xu2019sample}
P.~Xu, F.~Gao, and Q.~Gu.
\newblock Sample efficient policy gradient methods with recursive variance reduction.
\newblock In \emph{International Conference on Learning Representations}, 2019.

\bibitem[Xu et~al.(2020)Xu, Gao, and Gu]{xu2020improved}
P.~Xu, F.~Gao, and Q.~Gu.
\newblock An improved convergence analysis of stochastic variance-reduced policy gradient.
\newblock In \emph{Uncertainty in Artificial Intelligence}, pages 541--551, 2020.

\bibitem[Xu et~al.(2025)Xu, Ganesh, Mondal, Bai, and Aggarwal]{xu2025global}
Y.~Xu, S.~Ganesh, W.~U. Mondal, Q.~Bai, and V.~Aggarwal.
\newblock Global convergence for average reward constrained mdps with primal-dual actor critic algorithm.
\newblock \emph{Advances in Neural Information Processing Systems}, 2025.

\bibitem[Yuan et~al.(2022)Yuan, Gower, and Lazaric]{Vanilla_PL_Yuan_21}
R.~Yuan, R.~M. Gower, and A.~Lazaric.
\newblock A general sample complexity analysis of vanilla policy gradient.
\newblock In \emph{Proceedings of The 25th International Conference on Artificial Intelligence and Statistics}, volume 151, pages 3332--3380, 2022.

\bibitem[Zhang et~al.(2021)Zhang, Zhang, and Maguluri]{NEURIPS2021_096ffc29}
S.~Zhang, Z.~Zhang, and S.~T. Maguluri.
\newblock Finite sample analysis of average-reward td learning and q-learning.
\newblock In \emph{Advances in Neural Information Processing Systems}, volume~34, pages 1230--1242, 2021.

\bibitem[Zhang and Xie(2023)]{pmlr-v195-zhang23b}
Z.~Zhang and Q.~Xie.
\newblock Sharper model-free reinforcement learning for average-reward markov decision processes.
\newblock In \emph{Proceedings of Thirty Sixth Conference on Learning Theory}, volume 195, pages 5476--5477, 12--15 Jul 2023.

\bibitem[Zurek and Chen(2024)]{zurek2024spanbased}
M.~Zurek and Y.~Chen.
\newblock Span-based optimal sample complexity for weakly communicating and general average reward {MDP}s.
\newblock In \emph{The Thirty-eighth Annual Conference on Neural Information Processing Systems}, 2024.

\end{thebibliography}
\bibliographystyle{abbrvnat}

\newpage
\section*{NeurIPS Paper Checklist}

\begin{enumerate}

\item {\bf Claims}
    \item[] Question: Do the main claims made in the abstract and introduction accurately reflect the paper's contributions and scope?
    \item[] Answer: 
    % \answerTODO{} % Replace by 
    \answerYes{}
    % , \answerNo{}, or \answerNA{}.
    \item[] Justification: The claims are demonstrated in the key results (mainly in Section \ref{subsec:main-result}), which match the description in the abstract and introduction.
    % \justificationTODO{}

\item {\bf Limitations}
    \item[] Question: Does the paper discuss the limitations of the work performed by the authors?
    \item[] Answer: \answerYes{} % Replace by \answerYes{}, \answerNo{}, or \answerNA{}.
    \item[] Justification:  The assumptions given in the paper give the limitations of this work. Further, future work direction in the conclusions describe another limitation of this work. %See Introduction, Discussion, and other sections.

\item {\bf Theory Assumptions and Proofs}
    \item[] Question: For each theoretical result, does the paper provide the full set of assumptions and a complete (and correct) proof?
    \item[] Answer: \answerYes{} % Replace by \answerYes{}, \answerNo{}, or \answerNA{}.
    \item[] Justification: We have provided the assumptions used in the work in Section~\ref{sec:assump} or referenced in the theorem statements. All claims are backed with complete proofs in the Appendix.  %See assumptions.

    \item {\bf Experimental Result Reproducibility}
    \item[] Question: Does the paper fully disclose all the information needed to reproduce the main experimental results of the paper to the extent that it affects the main claims and/or conclusions of the paper (regardless of whether the code and data are provided or not)?
    \item[] Answer: \answerNA{} % Replace by \answerYes{}, \answerNo{}, or \answerNA{}.
    \item[] Justification: The experimental results are not provided, since the focus of the paper is on theoretical sample complexity analysis. 

\item {\bf Open access to data and code}
    \item[] Question: Does the paper provide open access to the data and code, with sufficient instructions to faithfully reproduce the main experimental results, as described in supplemental material?
    \item[] Answer: \answerNA{} % Replace by \answerYes{}, \answerNo{}, or \answerNA{}.
    \item[] Justification: The paper does not include experiments requiring code.

\item {\bf Experimental Setting/Details}
    \item[] Question: Does the paper specify all the training and test details (e.g., data splits, hyperparameters, how they were chosen, type of optimizer, etc.) necessary to understand the results?
    \item[] Answer: \answerNA{} % Replace by \answerYes{}, \answerNo{}, or \answerNA{}.
    \item[] Justification: Our work is theoretical, and hence there are no experiments.

\item {\bf Experiment Statistical Significance}
    \item[] Question: Does the paper report error bars suitably and correctly defined or other appropriate information about the statistical significance of the experiments?
    \item[] Answer: \answerNA{} % Replace by \answerYes{}, \answerNo{}, or \answerNA{}.
    \item[] Justification: The paper does not include experiments.

\item {\bf Experiments Compute Resources}
    \item[] Question: For each experiment, does the paper provide sufficient information on the computer resources (type of compute workers, memory, time of execution) needed to reproduce the experiments?
    \item[] Answer: \answerNA{} % Replace by \answerYes{}, \answerNo{}, or \answerNA{}.
    \item[] Justification: The paper does not include experiments.
    
\item {\bf Code Of Ethics}
    \item[] Question: Does the research conducted in the paper conform, in every respect, with the NeurIPS Code of Ethics \url{https://neurips.cc/public/EthicsGuidelines}?
    \item[] Answer: \answerYes{} % Replace by \answerYes{}, \answerNo{}, or \answerNA{}.
    \item[] Justification: All the points mentioned in the NeurIPS Code of Ethics are taken into consideration.

\item {\bf Broader Impacts}
    \item[] Question: Does the paper discuss both potential positive societal impacts and negative societal impacts of the work performed?
    \item[] Answer: \answerNA{} % Replace by \answerYes{}, \answerNo{}, or \answerNA{}.
    \item[] Justification: There are no specific societal impacts of this work beyond that the algorithms may have impact in the reinforcement learning algorithms in different applications.
    
\item {\bf Safeguards}
    \item[] Question: Does the paper describe safeguards that have been put in place for responsible release of data or models that have a high risk for misuse (e.g., pretrained language models, image generators, or scraped datasets)?
    \item[] Answer: \answerNA{} % Replace by \answerYes{}, \answerNo{}, or \answerNA{}.
    \item[] Justification: There is no release of data or models. 

\item {\bf Licenses for existing assets}
    \item[] Question: Are the creators or original owners of assets (e.g., code, data, models), used in the paper, properly credited and are the license and terms of use explicitly mentioned and properly respected?
    \item[] Answer: \answerNA{} % Replace by \answerYes{}, \answerNo{}, or \answerNA{}.
    \item[] Justification: There is no  code, data, models, etc. involved. 

\item {\bf New Assets}
    \item[] Question: Are new assets introduced in the paper well documented and is the documentation provided alongside the assets?
    \item[] Answer: \answerNA{} % Replace by \answerYes{}, \answerNo{}, or \answerNA{}.
    \item[] Justification:  There is no  code, data, models, etc. involved in the paper. 

\item {\bf Crowdsourcing and Research with Human Subjects}
    \item[] Question: For crowdsourcing experiments and research with human subjects, does the paper include the full text of instructions given to participants and screenshots, if applicable, as well as details about compensation (if any)? 
    \item[] Answer: \answerNA{} % Replace by \answerYes{}, \answerNo{}, or \answerNA{}.
    \item[] Justification: The paper does not involve crowdsourcing nor research with human subjects.

\item {\bf Institutional Review Board (IRB) Approvals or Equivalent for Research with Human Subjects}
    \item[] Question: Does the paper describe potential risks incurred by study participants, whether such risks were disclosed to the subjects, and whether Institutional Review Board (IRB) approvals (or an equivalent approval/review based on the requirements of your country or institution) were obtained?
    \item[] Answer: \answerNA{} % Replace by \answerYes{}, \answerNo{}, or \answerNA{}.
    \item[] Justification: The paper does neither involve crowd-sourcing nor research with human subjects.
    \item {\bf Declaration of LLM usage}
    \item[] Question: Does the paper describe the usage of LLMs if it is an important, original, or non-standard component of the core methods in this research? Note that if the LLM is used only for writing, editing, or formatting purposes and does not impact the core methodology, scientific rigorousness, or originality of the research, declaration is not required.
    %this research? 
    \item[] Answer: \answerNA{} % Replace by \answerYes{}, \answerNo{}, or \answerNA{}.
    \item[] Justification: We do not use LLM in our paper.
\end{enumerate}
\newpage
\setcounter{section}{0}
\renewcommand{\thesection}{\Alph{section}}
\vspace{-3cm}

\section{Discussion on the Discounted to Average Method}
\label{app:disc-avg}

We note that while there are several sample complexity analyses for discounted policy gradient methods, most treat the term $(1 - \gamma)^{-1}$ as a constant. As a result, most results exhibit an overly pessimistic dependence on this factor, since it is not the main focus of these works. However, achieving an optimal dependence on $(1 - \gamma)^{-1}$ is essential to obtain optimal regret or convergence guarantees when using the discounted-to-average reduction. This reduction is particularly sensitive to the $(1 - \gamma)^{-1}$ factor.

We now compare the sample complexity, a closely related notion to regret, of existing policy gradient algorithms when applied to the average-reward MDP (AMDP) setting. We focus on sample complexity since it aligns with the typical form of guarantees in the discounted setting and facilitates direct comparison. To translate results from the discounted MDP (DMDP) to the AMDP setting, we use Theorem 1 from \citep{wang2022near}, which was also employed in \citep{zurek2024spanbased} to establish optimal sample complexity for average-reward MDPs under a generative model. This reduction bound shows that if the DMDP with discount factor $\gamma$ is solved to an accuracy of $\epsilon_{\gamma} = \cO((1-\gamma)^{-1}\epsilon )$, the
resulting policy will be of $\cO(\epsilon )$ accuracy in the original AMDP. 

Table~\ref{table:disc_PG} summarizes the sample complexity bounds for discounted policy gradient methods. The best guarantee among these is $\cO(\epsilon^{-5})$, which remains significantly worse than the optimal $\cO(\epsilon^{-2})$ rate. This is akin to a regret bound of order $\cO(T^{4/5})$. Although some actor-critic methods, such as that of \citep{xu2020improved}, offer improved sample complexity guarantees in terms of $(1 - \gamma)^{-1}$, they assume ergodicity even in the discounted setting and are therefore not considered.

\begin{table*}[h]
	\caption{
Overview of global optimality convergence for discounted PG algorithms. DMDP Sample Complexity: Number of samples to achieve $J_{\gamma}^* - \E{J_{\gamma}(\theta_T)} \leq \frac{\varepsilon + \sqrt{\varepsilon_{\mathrm{bias}}}}{1-\gamma}$. AMDP Sample Complexity: Number of samples to achieve $J^* - \E{J(\theta_T)} \leq \varepsilon + \sqrt{\varepsilon_{\mathrm{bias}}}$. 
 }
    \label{table:disc_PG}
	\centering
	\begin{tabular}{|c|c|c|}
		\hline
          Algorithm & DMDP Sample Complexity & AMDP Sample Complexity  \\
		\hline \hline
		Vanilla-PG \citep{Vanilla_PL_Yuan_21}  & $(1-\gamma)^{-2} \varepsilon^{-3}$ & $ \varepsilon^{-5}$ \\
	\hline	
      Natural-PG \citep{liu2020improved}  & $(1-\gamma)^{-4}\varepsilon^{-3} $& $ \varepsilon^{-7}$\\
		\hline
		SRVR-PG \citep{liu2020improved}  & $ (1-\gamma)^{-2.5} \varepsilon^{-3}  $ & $ \varepsilon^{-5.5}$\\
		\hline
		STORM-PG \citep{ding-et-al22}  & $(1-\gamma)^{-9} \varepsilon^{-3}$& $ \varepsilon^{-12}$ \\
		\hline
		SCRN   
		\citep{Masiha_SCRN_KL}  & $(1-\gamma)^{-5} \varepsilon^{-2.5}$ & $ \varepsilon^{-7.5}$\\
		\hline
        %\cline{2-8}
		N-PG-IGT \citep{fatkhullin2023stochastic}  & $(1-\gamma)^{-5}\varepsilon^{-2.5}$& $ \varepsilon^{-7.5}$  \\
		%\cline{2-8}
        \hline
		  (N)-HARPG \citep{fatkhullin2023stochastic}  &   $(1-\gamma)^{-4} \varepsilon^{-2}$&   $ \varepsilon^{-6}$ \\
		\hline
		  ANPG \citep{mondal2024improved}  &   $(1-\gamma)^{-4} \varepsilon^{-2}$&   $ \varepsilon^{-6}$ \\
          \hline
    \end{tabular}
\end{table*}

% \section{Preliminaries for Update Rules in Algorithm \ref{alg:acb}}
% \label{sec:noiseless}

\section{Policy Gradient Theorem}\label{sec:PG:proof}
\begin{lemma}[Policy Gradient theorem, \citep{sutton1999policy}] Let Assumption \ref{assump_mdp} hold and assume a differentiable policy parameterization. Then, the following holds for all parameters $\theta \in \Theta$.
   \begin{align}
       \nabla J(\theta) = \E_{(s,a) \sim \eta^{\pi_{\theta}}} [Q^{\pi_{\theta}}(s,a)\nabla \log \pi_{\theta}(a|s)]
   \end{align}
   \label{lem:PG}
\end{lemma}
\begin{proof}[Proof of Lemma \ref{lem:PG}]
We begin by noting that
\begin{align}
        V^{\pi_{\theta}}(s) = \sum_a \pi_{\theta}(a|s)Q^{\pi_{\theta}}(s,a).
    \end{align}
By differentiating both sides of the above equation with respect to $\theta$, we obtain  \begin{align}
\label{eq:diff-v-temp}
        \nabla V^{\pi_{\theta}}(s) = \sum_a \nabla \pi_{\theta}(a|s)Q^{\pi_{\theta}}(s,a) +   \pi_{\theta}(a|s) \nabla Q^{\pi_{\theta}}(s,a) 
    \end{align}
Recall that
\begin{align}
    Q^{\pi_{\theta}}(s,a)=r(s,a)+\sum_{s'}P(s'|s,a) V^{\pi_{\theta}}(s')-J(\theta)
\end{align}
In a similar manner to \eqref{eq:diff-v-temp}, taking the derivative of both sides of the above equation with respect to $\theta$ gives \begin{align}
\label{eq:diff-q-temp}
    \nabla Q^{\pi_{\theta}}(s,a)=\sum_{s'}P(s'|s,a) \nabla V^{\pi_{\theta}}(s')-\nabla J(\theta).
\end{align}
Substituting \eqref{eq:diff-q-temp} in \eqref{eq:diff-v-temp} yields
    \begin{align}
        \nabla V^{\pi_{\theta}}(s) = \sum_a \left[ \nabla \pi_{\theta}(a|s)Q^{\pi_{\theta}}(s,a) +   \pi_{\theta}(a|s) \left(\sum_{s'}P(s'|s,a) \nabla V^{\pi_{\theta}}(s')-\nabla J(\theta)\right)\right]
    \end{align}
Re-arranging the above equation, we arrive at the following expression for $\nabla J(\theta)$
    \begin{align}
       \nabla J(\theta) = \sum_a \left[ \nabla \pi_{\theta}(a|s)Q^{\pi_{\theta}}(s,a) +\sum_{s'}  \pi_{\theta}(a|s) P(s'|s,a) \nabla V^{\pi_{\theta}}(s')\right]-\nabla V^{\pi_{\theta}}(s) 
    \end{align}
Taking expectation over $s \sim d^{\pi_{\theta}}$ on both sides, we obtain
    \begin{align}
    \begin{split}
       \nabla J(\theta) &= \sum_{s,a}   \left[ d^{\pi_{\theta}}(s) \nabla \pi_{\theta}(a|s)Q^{\pi_{\theta}}(s,a) +\sum_{s'}   d^{\pi_{\theta}}(s)\pi_{\theta}(a|s)  P(s'|s,a) \nabla V^{\pi_{\theta}}(s')\right]\\
       &\quad -\sum_s d^{\pi_{\theta}}(s) \nabla V^{\pi_{\theta}}(s) 
    \end{split}
    \end{align}
Since $\sum_s  d^{\pi_{\theta}}(s)P^{\pi_{\theta}}(s,s')=d^{\pi_{\theta}}(s')$, we arrive at the following expression.
    \begin{align}
    \label{eq:grad-expression}
       \nabla J(\theta) = \sum_{s,a}   \left[ d^{\pi_{\theta}}(s) \nabla \pi_{\theta}(a|s)Q^{\pi_{\theta}}(s,a) \right]
    \end{align}
\end{proof}

\section{General Properties of Unichain Markov Decision Processes}
\label{app:unichain}

In this section, we provide detailed proofs of Lemmas \ref{lem:ergodic-avg} and \ref{lem:bias-variance-unichain} from the main paper, along with bounds on the value functions (Lemma \ref{lem:q-v-a-bound}) and show that the hitting time of the recurrent class remains bounded by $B$ with a high-probability (Lemma \ref{lem:prob-exponential}). We begin with a few useful preliminary results.

\begin{lemma}[Strong Markov Property, Theorem~1.2.5 of \citep{norris1998markov}]
Let \( \{s_t\}_{t \geq 0} \) be a discrete-time Markov chain on a countable state space with transition kernel \( P \). If \( T \) is a (possibly infinite) stopping time, then for all \( s \in \mathcal{S} \) and any bounded measurable function \( f \),
\[
\E \left[ f(s_{T + n}) \mid \mathcal{F}_T \right] = \E \left[ f(s_n) \mid s_T \right] \quad \text{a.s. on } \{ T < \infty \}, \quad \text{for all } n \geq 0.
\]
\end{lemma}

The following intermediate result provides a bound on the convergence rate of the empirical state distribution to the stationary distribution when restricted to the recurrent class. 

\begin{lemma} 
\label{lem:rosenthal}
    Let Assumption \ref{assump_mdp} hold and consider any $\theta \in \Theta$. Then, the following bound holds:  
\begin{equation}
    \norm{\frac{1}{t} \sum_{i=1}^{t} (P^{\pi_{\theta}})^i(s_0, \cdot) - d^{\pi_{\theta}}(\cdot)}_{\TV} \leq \frac{C_{\mathrm{tar}}}{t}, \quad \forall s_0 \in \mathcal{S}_R^{\theta}.
\end{equation} 
\end{lemma}

\begin{proof}Since $\mathcal{S}_R^{\theta}$ is closed and communicating, the Markov chain restricted to $\mathcal{S}_R^{\theta}$ remains a Markov chain and is, in fact, irreducible. The result then follows from Corollary 3 of \citep{Roberts01011997}, which applies to (possibly periodic) irreducible Markov chains.
\end{proof}
\subsection{Proof of Lemma \ref{lem:ergodic-avg}}
Note that
\begin{align}
\label{eq:intermediate-erg-avg}
     \norm{\frac{1}{t} \sum_{i=1}^{t} (P^{\pi_{\theta}})^i(s_0, \cdot) - d^{\pi_{\theta}}(\cdot)}_{\TV} &= \frac{1}{2} \norm{\frac{1}{t} \sum_{i=1}^{t} (P^{\pi_{\theta}})^i(s_0, \cdot) - d^{\pi_{\theta}}(\cdot)}_{1} \nonumber \\
     &= \frac{1}{2t} \norm{\sum_{i=1}^{t} [(P^{\pi_{\theta}})^i(s_0, \cdot) - d^{\pi_{\theta}}(\cdot)]}_{1} 
\end{align}
  Thus, it suffices to show that the quantity $$ \norm{\sum_{i=1}^{t} [(P^{\pi_{\theta}})^i(s_0, \cdot) - d^{\pi_{\theta}}(\cdot)]}_{1}$$ is bounded by $C$, for all $t\geq 1$. Note that this term can be broken down in the following manner.
\begin{align}
     &\norm{\sum_{i=1}^{t} (P^{\pi_{\theta}})^i(s_0, \cdot) - d^{\pi_{\theta}}(\cdot)}_{1} \nonumber\\
    &=\sum_{s \in \mathcal{S}}\left|\sum_{i=1}^{t} (P^{\pi_{\theta}})^i(s_0,s) - d^{\pi_{\theta}}(s)\right| \nonumber\\
    &\overset{(a)}{=}\underbrace{\sum_{s \in \mathcal{S}_R^{\theta}}\left|\sum_{i=1}^{t} (P^{\pi_{\theta}})^i(s_0,s) - d^{\pi_{\theta}}(s)\right|}_{T_1}+ \underbrace{\sum_{s \in (\mathcal{S}_R^{\theta})^c}\sum_{i=1}^{t} (P^{\pi_{\theta}})^i(s_0,s)}_{T_2},
\end{align} 

where $(a)$ follows since $d^{\pi_{\theta}}(s)=0$ for all $s \in (\mathcal{S}_R^{\theta})^c$. Let $\E_{s_0}^{\theta}$ denote the expectation over the Markov chain generated from $\pi_{\theta}$ starting from state $s_0$.  Consider term $T_2$ and observe that
\begin{align}
\begin{split}
   \sum_{s \in (\mathcal{S}_R^{\theta})^c}\sum_{i=1}^{t} (P^{\pi_{\theta}})^i(s_0,s)&=\E_{s_0}^{\theta} \left[\sum_{i=1}^t \sum_{s \in (\mathcal{S}_R^{\theta})^c} \mathbf{1}(s_i = s)\right] = \E_{s_0}^{\theta} \left[\sum_{i=1}^{t} \mathbf{1}(s_i \in (\mathcal{S}_R^{\theta})^c) \right]\\
   &=\E_{s_0}^{\theta}[\min \{t,T_{\theta}-1\}] \leq \E_{s_0}^{\theta}[T_{\theta}] \leq  C_{\mathrm{hit}}.
\end{split}
\end{align}
We adopt the convention that the sum $\sum_{i=j}^k s_i$ is defined to be zero whenever $j > k$, regardless of the values of the summands $s_i$. The term $T_2$ can be bounded as follows. 
\begin{align*}
   &\sum_{s \in \mathcal{S}_R^{\theta}}\left|\sum_{i=1}^{t} [(P^{\pi_{\theta}})^i(s_0,s) - d^{\pi_{\theta}}(s)]\right|=  \sum_{s \in \mathcal{S}_R^{\theta}}\left|\E_{s_0}^{\theta}\left[\sum_{i=1}^{t} [\mathbf{1}(s_i=s) - d^{\pi_{\theta}}(s)] \right]\right|\\
   &\leq \E_{\theta}\left[\sum_{s \in \mathcal{S}_R^{\theta}}\left|\E_{T_{\theta}}\left[\sum_{i=1}^{t} [\mathbf{1}(s_i=s) - d^{\pi_{\theta}}(s)]\right]\right| \right]\\
   &\leq \E_{\theta}\Bigg[\underbrace{\sum_{s \in \mathcal{S}_R^{\theta}}\left|\E_{T_{\theta}}\left[\sum_{i=1}^{T_{\theta}-1} [\mathbf{1}(s_i=s) - d^{\pi_{\theta}}(s)] \right]\right|}_{T_3} +\underbrace{\sum_{s \in \mathcal{S}_R^{\theta}}\E_{T_{\theta}}\left[\sum_{i=T_{\theta}}^{t} [\mathbf{1}(s_i=s) - d^{\pi_{\theta}}(s)]\right] }_{T_4}\Bigg],
\end{align*} 
where $\E_{T_{\theta}}$ denotes the conditional expectation given the hitting time $T_{\theta}$, and $\E_{\theta}$ denotes the expectation over the random variable $T_{\theta}$. Finally, consider term $T_3$ 
\begin{align*}
    \E_\theta\left[\sum_{s \in \mathcal{S}_R^{\theta}}\left|\E_{T_\theta}\left[\sum_{i=1}^{T_\theta-1} [\mathbf{1}(s_i=s) - d^{\pi_{\theta}}(s)] \right]\right| \right] &\overset{(a)}{=} \E_\theta\left[\sum_{s \in \mathcal{S}_R^{\theta}}\E_{T_\theta}\left[\sum_{i=1}^{T_\theta-1}  d^{\pi_{\theta}}(s) \right]  \right] \leq\E_{\theta}[T_{\theta}]\leq C_{\mathrm{hit}},
\end{align*}
where $(a)$ follows since, by definition, no state $s \in \mathcal{S}_R^{\theta}$ is visited before $T_\theta$.
$T_4$ can be bounded as
\begin{align*}
    \E_\theta\left[\sum_{s \in \mathcal{S}_R^{\theta}}\left|\E_{T_\theta}\left[\sum_{i=T_{\theta}}^{t} [\mathbf{1}(s_i=s) - d^{\pi_{\theta}}(s)]\right]\right| \right]&\overset{(a)}{=}\E_\theta\left[\norm{\sum_{i=1}^{t-T_{\theta}+1}[(P^{\pi_{\theta}})^i(s_{T_{\theta}+1},s) - d^{\pi_{\theta}}(s)]}_1\right] \\
    &\overset{(b)}{\leq} 2 C_{\mathrm{tar}},
\end{align*}
where $(a)$ follows from the Strong Markov property and $(b)$ follows from the fact that $s_{T_{\theta}}\in \mathcal{S}_R^\theta$ and using Lemma~\ref{lem:rosenthal}. Combining the bounds on $T_1$, $T_2$, $T_3$, and $T_4$ gives
\begin{align}
\label{eq:l1-average}
     \norm{\sum_{i=1}^{t} [(P^{\pi_{\theta}})^i(s_0, \cdot) - d^{\pi_{\theta}}(\cdot)]}_{1}  \leq 2(C_{\mathrm{tar}}+C_{\mathrm{hit}}),
\end{align}

which yields the desired result using \eqref{eq:intermediate-erg-avg}. 

\subsection{Value function bounds}

Recall that the value function $V^{\pi}$ and the action-value function $Q^{\pi}$ are not uniquely defined for average-reward infinite-horizon MDPs, even in the ergodic case. We therefore consider $V^{\pi_{\theta}}$ such that

\begin{align}
\label{eq:v-cond}
\sum_{s \in \mathcal{S}} d^{\pi_{\theta}}(s) V^{\pi_{\theta}}(s) = 0.
\end{align}

Under this condition, $V^{\pi_{\theta}}(s)$ and $Q^{\pi_{\theta}}(s, a)$ are given by
\begin{align}
    \label{def_v_pi_theta_s}
    \begin{split}
        V^{\pi_{\theta}}(s) &=\lim_{N\to\infty}\frac{1}{N}\sum_{T=1}^N\E_{\theta}\left[\sum_{t=0}^{T}(r(s_t,a_t)-J(\theta))\bigg\vert s_0=s\right]\\
        Q^{\pi_{\theta}}(s,a)&=\lim_{N\to\infty}\frac{1}{N}\sum_{T=1}^N\E_{\theta}\left[\sum_{t=0}^{T}(r(s_t,a_t)-J(\theta))\bigg\vert s_0=s,a_0=a\right]
    \end{split}
\end{align}
where $\mathbb{E}_{\theta}[\cdot]$ denotes expectation with respect to trajectories generated by the policy $\pi_{\theta}$. These definitions employ the Cesàro limit to account for potential periodicity, which was absent in the ergodic case.

An equivalent expression for $V^{\pi_{\theta}}(s)$ is
\begin{align}
    \begin{split}
        V^{\pi_{\theta}}(s) &= \lim_{N\to\infty}\frac{1}{N}\sum_{T=1}^N\left[\sum_{t=0}^{T} \left[(P^{\pi_{\theta}})^t(s, \cdot) - d^{\pi_{\theta}}\right]^{\top}r^{\pi_{\theta}}\right].
    \end{split}
\end{align}

We now provide bounds for $V^{\pi}$, $Q^{\pi}$, and the advantage function $A^{\pi} = Q^{\pi} - V^{\pi}$ under the condition in \eqref{eq:v-cond}.

\begin{lemma}
\label{lem:q-v-a-bound}
For all $\theta \in \Theta$ and $(s,a)\in\mathcal{S}\times\mathcal{A}$, the following bounds hold:
\begin{enumerate}[label=(\roman*)]
    \item $V^{\pi_{\theta}}(s) \leq 2(C_{\mathrm{hit}}+C_{\mathrm{tar}})$
    \item $Q^{\pi_{\theta}}(s,a) \leq (1+2(C_{\mathrm{hit}}+C_{\mathrm{tar}}))$
    \item $A^{\pi_{\theta}}(s,a) \leq (1+4(C_{\mathrm{hit}}+C_{\mathrm{tar}}))$
\end{enumerate}

\end{lemma}
\begin{proof}

Using Hölder's inequality and the bound on the reward function, we obtain the following
    \begin{align*}
       \left(\sum_{t=0}^{T}[(P^{\pi_{\theta}})^t(s,\cdot)-d^{\pi_{\theta}}]\right)^\top r^{\pi}\leq \norm{\sum_{t=0}^{T}[(P^{\pi_{\theta}})^t(s,\cdot)-d^{\pi_{\theta}}] }_1\|r^{\pi}\|_{\infty}\leq \norm{\sum_{t=0}^{T}[(P^{\pi_{\theta}})^t(s,\cdot)-d^{\pi_{\theta}} ]}_1.
     \end{align*}

Combining the above bound with \eqref{eq:l1-average} concludes $(i)$. Parts $(ii)$ and $(iii)$ now follow easily, since both $r(s, a)$ and, consequently, $J$ are confined to the interval $[0, 1]$, which gives
    \begin{align}
        Q^{\pi_\theta}(s,a) = r(s,a)-J^{\pi_\theta} + \E_{s' \sim P(\cdot|s,a)} [V^{\pi_\theta}(s')]\leq (1+2(C_{\mathrm{hit}}+C_{\mathrm{tar}}))
    \end{align}
and 
    \begin{align}
        A^{\pi_\theta}(s,a) = Q^{\pi_\theta}(s,a)-V^{\pi_\theta}(s) \leq (1+4(C_{\mathrm{hit}}+C_{\mathrm{tar}})).
    \end{align}
\end{proof}

\subsection{Proof of Lemma \ref{lem:bias-variance-unichain}}

For clarity and ease of reference, we re-state and prove parts $(i)$ and $(ii)$ of Lemma \ref{lem:bias-variance-unichain} separately. We also provide a sharper bound for the case where the Markov chain is restricted to the irreducible class.

\begin{lemma}[Bias of Markovian Sample Average]
\label{lem:bias-markov}
    Let Assumption~\ref{assump_mdp} hold, and let $f : \mathcal{S} \to \mathbb{R}^d$ satisfy $\|f(s)\| \leq C_f$ for all $s \in \mathcal{S}$, for some constant $C_f > 0$. Then, the following bound holds:
    \begin{align*}
   \norm{\E\left[\frac{1}{B}\sum_{i=1}^Bf(s_i)\right]-\mu}  \leq \frac{\sqrt{d}C_f C}{B} 
\end{align*}
where $\E$ denotes the expectation of the Markov chain $\{s_i\}$ induced by $\pi_{\theta}$ starting from any $s_0\in \mathcal{S}$ and $\mu = \E_{s \sim d^{\pi_{\theta}}}[f(s)]$. Furthermore, if $s_1\in \mathcal{S}_R^{\theta}$, then
 \begin{align*}
   \norm{\E\left[\frac{1}{B}\sum_{i=1}^Bf(s_i)\right]-\mu}  \leq \frac{\sqrt{d}C_f C_{\mathrm{tar}}}{B}. 
\end{align*}
\end{lemma}
\begin{proof}
Let $f_j:\mathcal{S}\to \R$ denote the $j^{th}$ co-ordinate of $f$, i.e., $f(s)=[f_1(s),f_2(s),\cdots,f_d(s)]^\top$. We express the $j^{th}$ co-ordinate of $\E \left[\frac{1}{B}\sum_{i=1}^Bf(s_i)\right] $, $\E \left[\frac{1}{B}\sum_{i=1}^Bf_j(s_i)\right] $, as
\begin{align}
    \E \left[\frac{1}{B}\sum_{i=1}^Bf_j(s_i)\right] = p^\top \left[\frac{1}{B}\sum_{i=1}^B P^i\right]f_j,
\end{align}
 
Similarly, $\mu_j$ can be written as
\begin{align}
    \mu_j =\E_{s \sim d} [f_j(s)] = (d^{\pi_{\theta}})^\top f_j.
\end{align}
Then, by using Hölder's inequality for each $j$
\begin{align}
    \left|\E_{s_1\sim p}\left[\frac{1}{B}\sum_{i=1}^Bf_j(s_i)\right]-\mu_j\right|^2 = \norm{p^\top \left[\frac{1}{B}\sum_{i=1}^B P^i\right]-(d^{\pi_{\theta}})^\top}_1^2\|f_j\|_{\infty}^2 \overset{(a)}{\leq} \frac{C_f^2 C^2}{B^2},
\end{align}
where $(a)$ follows from Lemma~\ref{lem:ergodic-avg}. If $s_1 \in \mathcal{S}_R^{\theta}$, then Lemma~\ref{lem:rosenthal} can be applied instead, replacing $C$ with $C_{\mathrm{tar}}$. The result then follows by noting that
\begin{align}
     \norm{\E_{s_1\sim p}\left[\frac{1}{B}\sum_{i=1}^Bf(s_i)\right]-\mu}= \sqrt{\sum_{j=1}^d \left|\E_{s_1\sim p}\left[\frac{1}{B}\sum_{i=1}^Bf_j(s_i)\right]-\mu_j\right|^2}.
\end{align}
\end{proof}

\begin{lemma}[Variance of Markovian Sample Average]
\label{lem:variance-markov}
      Let Assumption~\ref{assump_mdp} hold, and let $f : \mathcal{S} \to \mathbb{R}^d$ satisfy $\|f(s)\| \leq C_f$ for all $s \in \mathcal{S}$, for some constant $C_f > 0$. Then, the following bounds hold:
    \begin{align*}
    \E \norm{\left[\frac{1}{B}\sum_{i=1}^Bf(s_i)\right]-\mu}^2 \leq \frac{C_f^2 + 2\sqrt{d}C_f^2C }{B}
\end{align*}
where $\E$ denotes the expectation of the Markov chain $\{s_i\}$ induced by $\pi_{\theta}$ starting from any $s_0\in \mathcal{S}$ and $\mu = \E_{s \sim d^{\pi_{\theta}}}[f(s)]$.  Furthermore, if $s_1\in \mathcal{S}_R^{\theta}$, then
 \begin{align*}
    \E \norm{\left[\frac{1}{B}\sum_{i=1}^Bf(s_i)\right]-\mu}^2 \leq \frac{C_f^2 + 2\sqrt{d}C_f^2C_{\mathrm{tar}} }{B}
\end{align*}
\end{lemma}
Observe that
\begin{align}
\begin{split}
    \mathrm{Var}\left(\frac{1}{B}\sum_{i=1}^B f(s_i)\right) &= \E \norm{\frac{1}{B}\sum_{i=1}^B f(s_i)-\mu }^2\\
    &= \frac{\sum_{i=1}^B \E\|f(s_i)-\mu\|^2+2\sum_{i=1}^B\sum_{j=i+1}^B\E \langle f(s_i)-\mu, f(s_j)-\mu\rangle}{B^2} \\
    &= \frac{\sum_{i=1}^B \E\|f(s_i)-\mu\|^2+2\sum_{i=1}^B\E \langle f(s_i)-\mu, \E_i [\sum_{j=i+1}^B (f(s_j)-\mu)]\rangle}{B^2} \\
    &\leq \frac{\sum_{i=1}^B \E\|f(s_i)-\mu\|^2+2\sum_{i=1}^B \E\| f(s_i)-\mu\| \|\E_i [\sum_{j=i+1}^B (f(s_j)-\mu)]\|}{B^2} \\
    &\overset{(a)}{\leq} \frac{\sum_{i=1}^B \E\|f(s_i)-\mu\|^2+2\sum_{i=1}^B \sqrt{d}C_fC_{\mathrm{tar}}\E\| f(s_i)-\mu\| }{B^2}\\
    &\leq \frac{C_f^2 + 2\sqrt{d}C_f^2C}{B},
\end{split}
\end{align}
where $\E_i$ denotes the conditional expectation given $s_i$ and $(a)$ follows from Lemma~\ref{lem:bias-markov}. If $s_1 \in \mathcal{S}_R^{\theta}$, then $C$ can be replaced by $C_{\mathrm{tar}}$. 

\subsection{Bound on the Hitting time of the Recurrent Class}
\begin{lemma}
\label{lem:prob-exponential}
Let Assumption~\ref{assump_mdp} hold, and let $\{s_t\}_{t \geq 0}$ be the state sequence induced by policy $\pi_{\theta}$. Then, for any $B \geq 1$, the probability of hitting $\mathcal{S}_R^{\theta}$ within $B$ steps satisfies
\[
\Pr(T_{\theta} \leq B) \geq 1 - 2^{-\lfloor B / 2C_{\mathrm{hit}} \rfloor}.
\]
\end{lemma}
\begin{proof}

By Markov's inequality,  
\[
\Pr(T_\theta > 2C_{\mathrm{hit}}) \leq \frac{\E[T_\theta]}{2C_{\mathrm{hit}}} \leq \frac{1}{2},
\]  
so \( \Pr(T_\theta \leq 2C_{\mathrm{hit}}) \geq \frac{1}{2} \).  

Now partition the time horizon \( [0, B] \) into \( N := \lfloor B / 2C_{\mathrm{hit}} \rfloor \) disjoint intervals of length \( 2C_{\mathrm{hit}} \). By the Markov property and the uniform bound on the expected hitting time, the probability of not hitting \( \mathcal{S}_R^{\theta} \) in each interval is at most \( \frac{1}{2} \), regardless of the starting state. Therefore,  
\[
\Pr(T_\theta > B) \leq \left(\frac{1}{2}\right)^N = 2^{-\lfloor B / 2C_{\mathrm{hit}} \rfloor}.
\]

Taking the complement gives the desired bound:
\[
\Pr(T_\theta \leq B) \geq 1 - 2^{-\lfloor B / 2C_{\mathrm{hit}} \rfloor}.
\]
\end{proof}

\section{Analysis of a General Linear Recursion}
\label{app:general_linear}

Recall that the Critic and NPG subroutines in the algorithm can be viewed as instances of stochastic linear recursions. To analyze their behavior in a unified framework, we consider a general stochastic linear recursion that captures the essential structure of both subroutines. This approach allows us to derive bounds that apply to each case as a special instance. Specifically, we study the recursion of the form:
\begin{align}\label{eq:xt}
x_{h+1} = x_h - \bar{\beta} (\widehat{P}_h x_h - \widehat{q}_h),
\end{align}
where $\widehat{P}_h$ and $\widehat{q}_h$ are noisy estimates of $P \in \mathbb{R}^{n \times n}$ and $q \in \mathbb{R}^n$, respectively, and $h \in \{0, \dots, H-1\}$. Assume that the following bounds hold for all $h$:
\[
(A_1)\quad \E_h\left[\|\widehat{P}_h - P\|^2\right] \leq \sigma_P^2, \quad(A_2)\quad \|\E_h\left[\widehat{P}_h\right] - P\|^2 \leq \delta_P^2,
\]
\[
(A_3)\quad\E_h\left[\|\widehat{q}_h - q\|^2\right] \leq \sigma_q^2, \quad(A_4)\quad \|\E_h\left[\widehat{q}_h\right] - q\|^2 \leq \delta_q^2,
\]
and
\[
(A_5)\quad\|\E\left[\widehat{q}_h\right] - q\|^2 \leq \bar{\delta}_q^2.
\]
Here, $\E_h$ denotes the conditional expectation given the history up to step $h$. Since $\E[\widehat{q}_h] = \E[\E_h[\widehat{q}_h]]$, we have $\bar{\delta}_q^2 \leq \delta_q^2$. Additionally, assume:
\[
(A_6)\quad\|P\| \leq \Lambda_P, \quad(A_7)\quad  \|q\| \leq \Lambda_q,
\]
and
\[(A_8)\quad x^\top P x \geq \lambda_P \|x\|^2,\quad (A_9)\quad \ker( P) \subseteq \ker( \widehat{P}_h)  \]
for all $x \in \ker(P)^\perp$ and $h\geq1$. Let $x^* = P^\dagger q$, where $P^\dagger$ is the Moore-Penrose pseudoinverse of $P$ and $\Pi$ denote the projection operator onto the space $\ker(P)^{\perp}$, i.e., the orthogonal complement of the kernel (null space) of $P$. 

While a general linear analysis was conducted in \citep{ganesh2024order}, our approach differs significantly due to the fact that the critic matrix is not strictly positive definite. This necessitates the additional condition $(A_9)$, which was not required in the ergodic setting. Despite the lack of strict positive definiteness, we still achieve the same convergence rate.

\begin{theorem}
\label{thm_2}
Consider the recursion \eqref{eq:xt} and suppose $(A_1)-(A_9)$ hold. Further, let $\delta_P \leq \lambda_P / 8$ and $\bar{\beta} = \frac{\lambda_P}{\Lambda_P}$. Then, the following bounds hold for all $H\geq 1$:
\begin{align*}
(a)\quad\E\left[\|\Pi(x_H - x^*)\|^2\right] \leq &\cO\Bigg(\exp\left(-\tfrac{H \lambda_P^2}{4\Lambda_q}\right)R_0^2 +  \sigma_P^2 \lambda_P^{-2} \Lambda_q^2\Lambda_P^{-1} + \sigma_q^2 \Lambda_P^{-1}+ \delta_P^2 \lambda_P^{-4} \Lambda_q^2 +\lambda_P^{-2} \delta_q^2\Bigg).
\end{align*}
and
\begin{align*}
(b)\quad \|\Pi(\E[x_H] - x^*)\|^2 \leq \cO\Bigg(\exp\left(-\tfrac{H \lambda_P^2}{4\Lambda_q}\right) R_0^2 + \delta_P^2R_0^2\lambda_P^{-2}+\delta_P^2\Lambda_P^2\lambda_P^{-4}+\bar{\delta}_q^2\lambda_P^{-2}\Bigg),
\end{align*}
where $R_0=\|\Pi(x_0-x_*)\|^2$
\end{theorem}

\subsection{Proof of Theorem~\ref{thm_2}\textit{(a)}}

Before proceeding, we provide some useful results with their proofs.
\begin{lemma}
Consider a matrix $ A $ such that $ x^\top A x \geq \lambda\|x\|^2 $ for all $ x \in (\ker(A))^\perp $. Then, we have
$$\|A^\dagger\| \leq \frac{1}{\lambda}$$
\end{lemma}
\begin{proof}
It is known that the operator norm of $ A^\dagger $ is $ \|A^\dagger\| = 1 / \sigma_r $, where $ \sigma_r $ is the smallest nonzero singular value of $ A $ (see Theorem 5.5.1 in \cite{golub2013matrix}).

Note that $ x^\top A x = x^\top A^\top x  $, which in turn implies $ x^\top A x = (1/2) x^\top (A+A^\top) x  $.
The condition $ x^\top A x \geq \lambda \|x\|^2 $ implies that $x^\top \frac{(A+A^\top)}{2} x    \geq \lambda \|x\|^2 $ for all $ x \in (\ker(A))^\perp $. Since $\frac{(A+A^\top)}{2}$ is a symmetric matrix, this implies that all non-zero eigenvalues of $\frac{(A+A^\top)}{2}$ are bounded below by $\lambda$. Now using Prop. III.5.1 in \cite{bhatia1997matrix}, it follows that $ \sigma_r \geq \lambda$ which concludes the proof.

\end{proof}

\begin{lemma}
\label{lem:proj}
Let $\Pi$ denote the projection operator onto the space $\ker(P)^{\perp}$, i.e., the orthogonal complement of the kernel (null space) of $P$. Then the following properties hold:
\begin{enumerate}
    \item[(i)] $\Pi^\top \Pi = \Pi^2 = \Pi$,
    \item[(ii)] If $\widetilde{P}$ is such that $\ker(P) \subseteq \ker(\widetilde{P})$, then $\widetilde{P}\Pi = \widetilde{P}$.
\end{enumerate}
\end{lemma}

\begin{proof}  
The first property holds since $\Pi$ is a projection operator onto a linear subspace. Such operators are both symmetric and idempotent, which implies $\Pi^\top = \Pi$ and $\Pi^2 = \Pi$, hence $\Pi^\top \Pi = \Pi$.

For the second property, let $\Pi^{\perp}$ denote the projection operator onto $\ker(P)$. By the orthogonal decomposition of the space, we have
$$
\Pi + \Pi^{\perp} = I,
$$
where $I$ is the identity operator. For any $x$, it follows that
$$
\widetilde{P}x = \widetilde{P}(\Pi x + \Pi^{\perp} x) = \widetilde{P}\Pi x,
$$
where the last equality holds since $\Pi^{\perp} x \in \ker(P)\subseteq \ker(\widetilde{P})$ and thus $\widetilde{P}\Pi^{\perp} x = 0$. This establishes the result.
\end{proof}

Coming back to the proof, consider \eqref{eq:xt} and apply the projection operator $\Pi$ to both sides. Using the linearity of $\Pi$, we obtain the following:
\begin{align}
\begin{split}
    \Pi(x_{h+1} - x^*)  
    &= \Pi\left(x_h - x^* - \bar{\beta} (\widehat{P}_h x_h - \widehat{q}_h)\right) \\
    &= \Pi\left(x_h - x^* - \bar{\beta} (P x_h - q + M_{h+1})\right) \\
    &= \Pi\left(x_h - x^* - \bar{\beta} P (x_h - x^*) + \bar{\beta} M_{h+1}\right) \\
    &= \Pi(x_h - x^*) - \bar{\beta} \Pi P (x_h - x^*) + \bar{\beta} \Pi M_{h+1}, 
\end{split}
\end{align}
where $M_{h+1} = (\widehat{P}_h - P)x_h - (\widehat{q}_h - q)$.

Taking the square norm on both sides of the above equation and then taking expectation we obtain
\begin{align}
\label{eq:T-123}
\begin{split}
    &\E \|\Pi(x_{h+1} - x^*)\|^2 \\  
     &= \E \|\Pi(x_h - x^*) - \bar{\beta} \Pi P (x_h - x^*) + \bar{\beta} \Pi M_{h+1}   \|^2 \\  
     &= \underbrace{\E \|\Pi(x_h - x^*) - \bar{\beta} \Pi P (x_h - x^*)\|^2}_{T_1} + \underbrace{\beta^2\E \| \Pi M_{h+1}   \|^2}_{T_2} \\
     &\quad + \underbrace{ 2\E \langle \Pi(x_h - x^*) - \bar{\beta} \Pi P (x_h - x^*), \bar{\beta} \Pi M_{h+1}  \rangle }_{T_3}
\end{split}
\end{align}
We now analyze term $T_1$
\begin{align}
\begin{split}
    \E \|\Pi&(x_h - x^*) - \bar{\beta} \Pi P (x_h - x^*)\|^2 \\
    &= \E \left[\|\Pi(x_h - x^*) \|^2 -2\bar{\beta} \langle \Pi(x_h - x^*),  \Pi P (x_h - x^*) \rangle + \bar{\beta}^2 \| \Pi P (x_h - x^*)\|^2\right]\\
  &= \E \left[ \|\Pi(x_h - x^*) \|^2 -2\bar{\beta}  (x_h - x^*)^\top \Pi^\top  \Pi P (x_h - x^*)  + \bar{\beta}^2 \| \Pi P (x_h - x^*)\|^2\right]\\   
    &\overset{(a)}{=}\E \left[  \|\Pi(x_h - x^*) \|^2 -2\bar{\beta}  (x_h - x^*)^\top \Pi^\top P \Pi (x_h - x^*)  + \bar{\beta}^2 \| \Pi P (x_h - x^*)\|^2\right]\\  
     &= \E \left[\|\Pi(x_h - x^*) \|^2 -2\bar{\beta}  (\Pi (x_h - x^*))^\top P (\Pi (x_h - x^*))  + \bar{\beta}^2\| \Pi P (x_h - x^*)\|^2\right]\\ 
     &\overset{(b)}{=} (1-\lambda_P \bar{\beta} + \bar{\beta}^2 \Lambda_P)\E \|\Pi(x_h - x^*) \|^2,
\end{split}
\end{align}
where $(a)$ follows from $(A_9)$ and Lemma \ref{lem:proj} and $(b)$ follows from $(A_6)$ and $(A_8)$. Now, consider term $T_3$
\begin{align}
\begin{split}
    &\E \langle \Pi(x_h - x^*) - \bar{\beta} \Pi P (x_h - x^*), \bar{\beta} \Pi M_{h+1}  \rangle \\
    &\quad = \E \langle \Pi(x_h - x^*) - \bar{\beta} \Pi P (x_h - x^*), \bar{\beta} \E[\Pi M_{h+1}|x_h]  \rangle \\
    &\quad \overset{(a)}{\leq}  \frac{\bar{\beta} \lambda_P}{2}\E \|\Pi(x_h - x^*) - \bar{\beta} \Pi P (x_h - x^*)\|^2 + \frac{2\bar{\beta} }{\lambda_P} \E \| \Pi \E [M_{h+1}|x_h]\|^2\\
\end{split}
\end{align}
where $(a)$ follows from the linearity of $\Pi$, which allows it to be exchanged with the conditional expectation, and from applying Young's inequality.
Adding $T_1$ and $T_3$ gives
\begin{align}
\begin{split}
    & \E \|\Pi(x_h - x^*) - \bar{\beta} \Pi P (x_h - x^*)\|^2+ \E \langle \Pi(x_h - x^*) - \bar{\beta} \Pi P (x_h - x^*), \bar{\beta} \Pi M_{h+1}  \rangle \\
    &\leq \left( 1+\frac{\bar{\beta} \lambda_P}{2}\right)\E \|\Pi(x_h - x^*) - \bar{\beta} \Pi P (x_h - x^*)\|^2 + \frac{2\bar{\beta} }{\lambda_P} \| \Pi \E [M_{h+1}|x_h]\|^2\\
    &= \left( 1 + \frac{\bar{\beta} \lambda_P}{2}\right) \left(1-\lambda_P \beta+\beta^2 \Lambda_P\right)\E \|\Pi(x_h - x^*)\|^2 + \frac{2\bar{\beta} }{\lambda_P} \| \Pi \E [M_{h+1}|x_h]\|^2\\
    &= \left( 1 - \frac{\bar{\beta} \lambda_P}{2}+\frac{\bar{\beta}^2(2\Lambda_P- \lambda_P^2)}{2} +\frac{\bar{\beta}^3 \Lambda_P\lambda_P}{2}\right) \E \|\Pi(x_h - x^*)\|^2 + \frac{2\bar{\beta} }{\lambda_P} \| \Pi \E [M_{h+1}|x_h]\|^2
\end{split}
\end{align}

Consider the term $\Pi \E [M_{h+1}|x_h]$ in the above sum. 
\begin{align*}
     \Pi \E [M_{h+1}|x_h] &=\Pi ( \E [(\widehat{P}_h - P)x_h - (\widehat{q}_h - q)|x_h]) \\
     &=\Pi \E[(\widehat{P}_h - P)|x_h]x_h - \Pi \E [(\widehat{q}_h - q)|x_h]\\
     &=\Pi \E[(\widehat{P}_h - P)|x_h] (x_h-x^*)+\Pi \E[(\widehat{P}_h - P)|x_h]x^* - \Pi \E [(\widehat{q}_h - q)|x_h]\\
      &\overset{(a)}{=} \Pi \E[(\widehat{P}_h - P)|x_h] \Pi (x_h -x^*) +\Pi \E[(\widehat{P}_h - P)|x_h]x^* - \Pi \E [(\widehat{q}_h - q)|x_h]
\end{align*}
where $(a)$ follows from $(A_9)$ and Lemma \ref{lem:proj}. Using the above bound, we obtain the following using triangle inequality
\begin{align*}
     \E \|\Pi \E [M_{h+1}|x_h]\|^2 &\leq 3\E[\|\Pi \E[(\widehat{P}_h - P)|x_h] \Pi (x_h -x^*)\|^2] + 3\E[\|\Pi \E[(\widehat{P}_h - P)|x_h]x^*\|^2] \\
     &\quad+ 3\E[\|\Pi \E [(\widehat{q}_h - q)|x_h]\|^2]\\
     &\leq 3\E[\|\E[(\widehat{P}_h - P)|x_h]\|^2\| \Pi (x_h -x^*)\|^2] + 3\E [\| \E[(\widehat{P}_h - P)|x_h]\|^2\|x^*\|^2]\\
     &\quad + 3\E[\|\E [(\widehat{q}_h - q)|x_h]\|^2]\\
     &\overset{(a)}{\leq} 3 \delta_P^2 \E \| \Pi (x_h -x^*)\|^2 + 3 \delta_P^2 \lambda_P^{-2} \Lambda_q^2+ 3 \delta_q^2
\end{align*}
where, $(a)$ follows from $(A_2)$ and $(A_4)$. Similarly,
we analyze the term $T_2$ below.
\begin{align}
\begin{split}
    \bar{\beta}^2&\E \| \Pi M_{h+1}   \|^2 \\
    &\leq \bar{\beta}^2\E \| M_{h+1}   \|^2 \\
    &= \bar{\beta}^2 \E \|(\widehat{P}_h - P)x_h - (\widehat{q}_h - q)\|^2 \\
    &= \bar{\beta}^2 \E \|(\widehat{P}_h - P)\Pi(x_h-x^*)+(\widehat{P}_h - P)\Pi x^* - (\widehat{q}_h - q)\|^2\\
    &= 3\bar{\beta}^2 \left(\E[\E [\|\widehat{P}_h - P\|^2|x_h]\|\Pi (x_h-x^*)\|^2]+\E \|\widehat{P}_h - P\|^2\|x^*\|^2 + \E\|\widehat{q}_h - q\|^2\right)\\
    &\overset{(a)}{\leq} 3\bar{\beta}^2 \sigma_P^2\E[\|\Pi (x_h-x^*)\|^2]+ 3 \bar{\beta}^2 \sigma_P^2\lambda_P^{-2}\Lambda_q^2 + 3 \bar{\beta}^2\sigma_q^2 
\end{split}
\end{align}
where $(a)$ follows from $(A_1)$ and $(A_3)$. Combining bounds obtained for $T_1$, $T_2$ and $T_3$ with \eqref{eq:T-123}, we obtain the following bound.
\begin{align}
\begin{split}
    &\E \|\Pi(x_{h+1} - x^*)\|^2 \\
    &\leq \left(1 - \frac{\bar{\beta} ( \lambda_P-12\delta_P^2\lambda_P^{-1})}{2}+\frac{\bar{\beta}^2(3\sigma_P^2+2\Lambda_P - \lambda_P^2)}{2} +\frac{\bar{\beta}^3 \Lambda_P \lambda_P}{2}\right)\E \|\Pi (x_h-x^*)\|^2 \\
    &+ 3 \bar{\beta}^2 \sigma_P^2\lambda_P^{-2}\Lambda_q^2+ 3 \bar{\beta}^2\sigma_q^2 + 6\bar{\beta} \delta_P^2 \lambda_P^{-3} \Lambda_q^2 + 6\bar{\beta} \lambda_P^{-1} \delta_q^2 
\end{split}
\end{align}
Suppose $12\delta_P^2 \leq  \lambda_P^2/2$ and $\bar{\beta}$ is small enough such that $\bar{\beta} \leq 1$ and $\bar{\beta}  \leq \frac{\lambda_P}{2\left((3\sigma_P^2+2\Lambda_P- \lambda_P^2) +\Lambda_P\lambda_P\right)}$. Then, the above bound gives us
\begin{align}
    \E \|\Pi(x_{h+1} - x^*)\|^2 \leq \left(1- \frac{\bar{\beta} \lambda_P}{4}\right)\E \|\Pi(x_{h} - x^*)\|^2 +\epsilon_{h},
\end{align}
where 
\begin{align}
    \epsilon_{h} \coloneqq 3 \bar{\beta}^2 \sigma_P^2\lambda_P^{-2}\Lambda_q^2+ 3 \bar{\beta}^2\sigma_q^2 + 6\bar{\beta} \delta_P^2 \lambda_P^{-3} \Lambda_q^2 + 6\bar{\beta} \lambda_P^{-1} \delta_q^2 .
\end{align}
Unrolling the recursion yields
\begin{align}
    \E \|\Pi(x_{h+1} - x^*)\|^2 &\leq \left(1- \frac{\bar{\beta} \lambda_P}{4}\right)^{H}\E \|\Pi(x_{1} - x^*)\|^2 +\sum_{h=1}^H \left(1- \frac{\bar{\beta} \lambda_P}{4}\right)^{H-h+1} \epsilon_{h}\\
    &\leq \mathrm{exp}\left(- \frac{H \bar{\beta} \lambda_P}{4}\right)\E \|\Pi(x_{1} - x^*)\|^2 + \frac{4 \epsilon_{h} }{\bar{\beta} \lambda_P}
\end{align}
Setting $\bar{\beta} = \frac{\lambda_P}{\Lambda_P}$ in the above equation yields the final result.

\subsection{Proof of Theorem \ref{thm_2}\textit{(b)}}

We analyze the squared bias term $\|\Pi(\mathbb{E}[x_h] - x^*)\|^2$, where $x_h$ denotes the iterate at step $h$, $x^*$ is the optimal solution, and $\Pi$ is a projection operator. Recall that the update rule for $x_{h+1}$ is can be written as:
$$
x_{h+1} = x_h -\bar{\beta} P (x_h - x^*) +\bar{\beta} M_{h+1},
$$

Taking expectation and applying the projection operator $\Pi$ on both sides, we obtain:
$$
\|\Pi(\mathbb{E}[x_{h+1}] - x^*)\|^2 = \|\Pi \mathbb{E}[(x_h - x^*) -\bar{\beta} P(x_h - x^*) +\bar{\beta} M_{h+1}]\|^2.
$$

Expanding the norm using the properties of inner products yields:
\begin{align}
\begin{split}
\|\Pi(\mathbb{E}[x_{h+1}] - x^*)\|^2 &= \|\Pi \mathbb{E}[(x_h - x^*) -\bar{\beta} P (x_h - x^*)]\|^2 +\bar{\beta}^2 \|\Pi \mathbb{E}[M_{h+1}]\|^2 \\
&- 2 \langle \Pi \mathbb{E}[(x_h - x^*) -\bar{\beta} P (x_h - x^*)], \bar{\beta} \Pi \mathbb{E}[M_{h+1}] \rangle.
\end{split}
\end{align}

Using Young's inequality, we bound the cross-term as follows:
\begin{align*}
-2 \langle \Pi \mathbb{E}[(x_h - x^*) - \bar{\beta} P(x_h - x^*)], \bar{\beta} \Pi \mathbb{E}[M_{h+1}] \rangle &\leq \frac{\bar{\beta} \lambda_P}{2} \|\Pi \mathbb{E}[(x_h - x^*) - \bar{\beta} P (x_h - x^*)]\|^2\\
&+ \frac{2\bar{\beta}}{\lambda_P} \|\Pi \mathbb{E}[M_{h+1}]\|^2.
\end{align*}

Substituting this bound back, we find
$$
\|\Pi(\mathbb{E}[x_{h+1}] - x^*)\|^2 \leq \left(1 + \frac{\bar{\beta} \lambda_P}{2}\right) \|\Pi \mathbb{E}[(x_h - x^*) - \bar{\beta} P (x_h - x^*)]\|^2 + \left(\frac{2\bar{\beta}}{\lambda_P} + \bar{\beta}^2\right) \|\Pi \mathbb{E}[M_{h+1}]\|^2.
$$

Next, we analyze the term $\|\Pi \mathbb{E}[(x_h - x^*) - \bar{\beta} P (x_h - x^*)]\|^2$. Using the linearity of expectation, we have
$$
\|\Pi \mathbb{E}[(x_h - x^*) - \bar{\beta} P (x_h - x^*)]\|^2 = \|\Pi (\mathbb{E}[x_h] - x^*) - \bar{\beta} \Pi P (\mathbb{E}[x_h] - x^*)\|^2.
$$

Expanding this norm
\begin{align*}
\|\Pi (\mathbb{E}[x_h] - x^*) - \bar{\beta} P (\mathbb{E}[x_h] - x^*)\|^2 &= \|\Pi (\mathbb{E}[x_h] - x^*)\|^2 + \bar{\beta}^2 \|P (\mathbb{E}[x_h] - x^*)\|^2 \\
&- 2\bar{\beta} \langle \Pi (\mathbb{E}[x_h] - x^*), P \Pi (\mathbb{E}[x_h] - x^*) \rangle.
\end{align*}

Since $ \Pi (\mathbb{E}[x_h] - x^*) \in \ker(P)^{\perp}$, we can use $A_6$ and $A_8$ to obtain
$$
\|\Pi (\mathbb{E}[x_h] - x^*) - \bar{\beta} P (\mathbb{E}[x_h] - x^*)\|^2 \leq (1 - \bar{\beta} \lambda_P+ \bar{\beta}^2 \Lambda_P) \|\Pi (\mathbb{E}[x_h] - x^*)\|^2.
$$

Finally, we bound the noise term $\|\Pi \mathbb{E}[M_{h+1}]\|^2$. Using the conditional expectation property, we expand
$$
\|\Pi \mathbb{E}[M_{h+1}]\|^2 = \|\Pi \mathbb{E}[\mathbb{E}[M_{h+1} | x_h]]\|^2.
$$

Decomposing $M_{h+1}$ into its components
$$
M_{h+1} = (\widehat{P}_h - P)(x_h - x^*) + (\widehat{P}_h - P)x^* + (\widehat{q}_h - q),
$$
we bound each term separately. Using $A_2$ and $A_5$, we find
$$
\|\Pi \mathbb{E}[M_{h+1}]\|^2 \leq 3 \delta_P^2 R_0^2 + 3 \delta_P^2 \Lambda_P^2 \lambda_P^{-2} + 3 \bar{\delta}_q^2,
$$
where $R_0^2 = \mathbb{E}[\|\Pi (x_h - x^*)\|^2]$.

Combining all results, the squared bias term satisfies the recurrence:
\begin{align*}
\|\Pi(\mathbb{E}[x_{h+1}] - x^*)\|^2 &\leq \left(1 + \frac{\bar{\beta} \lambda_P}{2}\right)(1 - \bar{\beta} \lambda_P+ \bar{\beta}^2 \Lambda_P) \|\Pi(\mathbb{E}[x_h] - x^*)\|^2 + \bar{\epsilon}_h\\
&=\left( 1 - \frac{\bar{\beta} \lambda_P}{2}+\frac{\bar{\beta}^2(2\Lambda_P- \lambda_P^2)}{2} +\frac{\bar{\beta}^3 \Lambda_P\lambda_P}{2}\right)  \|\Pi(\mathbb{E}[x_h] - x^*)\|^2 + \bar{\epsilon}_h
\end{align*}
where
$$
\bar{\epsilon}_h = \left(\frac{2\bar{\beta}}{\lambda_P} + \bar{\beta}^2\right) \left(3 \delta_P^2 R_0^2 + 3 \delta_P^2 \Lambda_P^2 \lambda_P^{-2} + 3 \bar{\delta}_q^2\right).
$$
Suppose $\beta$ is such that $ \bar{\beta} < (2\Lambda_P)^{-1}$. Then, the above bound gives us
\begin{align}
\|\Pi(\mathbb{E}[x_{h+1}] - x^*)\|^2 &\leq \left( 1 - \frac{\bar{\beta} \lambda_P}{4}\right)  \|\Pi(\mathbb{E}[x_h] - x^*)\|^2 + \bar{\epsilon}_h
\end{align}
Unrolling this recursion over $h$ steps, we obtain:
\begin{align}
\|\Pi(\mathbb{E}[x_{H}] - x^*)\|^2 &\leq \left( 1 - \frac{\bar{\beta} \lambda_P}{4}\right)^H  \|\Pi(x_0 - x^*)\|^2 + \frac{\bar{\epsilon}_h}{\bar{\beta}\lambda_P}
\end{align}
Setting $\bar{\beta} = \frac{\lambda_P}{\Lambda_P}$ in the above equation yields the final result.

\section{Proof of Theorem \ref{thm:critic-final}}
\label{app:critic}
Recall the definition of $Z_\theta$
\begin{align}
Z_{\theta} \coloneqq \left\{ z \,\middle|\, [\Phi z](i) = [\Phi z](j) \quad \forall i,j \in \mathcal{S}_R^{\theta} \right\}.
\end{align} 

We now state and prove a few useful results for the critic analysis. Recall that $M_{\theta} \coloneqq \mathbb{E}_{\theta} \left[ \phi(s) \left( \phi(s) - \phi(s') \right)^\top \right]$ constitutes a submatrix of the critic matrix $A_v(\theta)$. We begin by characterizing the kernel of $M_{\theta}$ in Lemma~\ref{lem:critic_psd}.

\begin{lemma}
\label{lem:critic_psd}
Let Assumption \ref{assump_mdp} hold, and suppose $\Phi$ is full rank. Then $\ker(M_{\theta}) = Z_{\theta}$.
\end{lemma}
\begin{proof}
Let $ f : \mathcal{S} \to \mathbb{R} $ be an arbitrary function. Consider the quantity
$$
\E\left[(f(s) - f(s'))^2\right],
$$
where the expectation is over $ s \sim d^{\pi_{\theta}} $ and $ s' \sim P^{\pi_{\theta}}(\cdot \mid s) $ is the next state under the transition kernel induced by policy $ \pi_{\theta} $. This expression captures the expected squared one-step difference in the values of $ f $ along the Markov chain.

It is known that if the chain is irreducible and $ f $ is non-constant, then $ \E_{s \sim d^{\pi_{\theta}}} \left[(f(s) - f(s'))^2\right] > 0 $ \citep{roy1997}. We extend this result to the setting where the chain induced by $ \pi_{\theta} $ has a single recurrent class (possibly with transient states). 

Suppose $ f $ is non-constant on this recurrent class. Then, we claim that $ \E_{s \sim d^{\pi_{\theta}}} \left[(f(s) - f(s'))^2\right] > 0 $. It suffices to show that there exist states \( s, s' \) within the recurrent class such that \( P^{\pi_{\theta}}(s, s') > 0 \) and \( f(s) \neq f(s') \). Suppose, for the sake of contradiction, that no such pair exists. Then, the recurrent class could be partitioned into disjoint subsets, each corresponding to a distinct constant value of \( f \), with no transitions between subsets. This, however, contradicts the fact that the recurrent class is closed and all states within it are communicating. Hence, such a pair \( (s, s') \) must exist, which implies that the expected squared difference is strictly positive since
\begin{align}
    \E_{s \sim d^{\pi_{\theta}}} \left[(f(s) - f(s'))^2\right] \geq  d^{\pi_{\theta}}(s)P^{\pi_{\theta}}(s, s') (f(s) - f(s'))^2 > 0.
\end{align}

Separately, note that $\E_{s \sim d^{\pi_{\theta}}} [(f(s)-f(s'))^2]$ can also be written as follows
    \begin{align}
    \begin{split}
        \E_{s \sim d^{\pi_{\theta}}} [(f(s)-f(s'))^2]  &= \E_{s \sim d^{\pi_{\theta}}} [f(s)^2-2f(s)f(s')+f(s')^2]\\
        &\overset{(a)}{=} 2 \E_{s \sim d^{\pi_{\theta}}} [f(s)(f(s)-f(s'))],
    \end{split}
    \end{align}
    where $(a)$ follows since $\sum_{s\in\mathcal{S}}d^{\pi_{\theta}}(s)P^{\pi_{\theta}}(s, s')=d^{\pi_{\theta}}(s')$  Taking $f = \Phi z$, it follows that $\E_{s \sim d^{\pi_{\theta}}} [([\Phi z](s)-[\Phi z](s'))^2]=z^\top M_{\theta} z$, which concludes the proof.
\end{proof}

The following result provides a lower bound on $\xi^\top A_v(\theta) \xi$ using Assumption \ref{assum:critic_positive_definite}.

\begin{lemma}
\label{lem:td-pd}
    For a large enough $c_{\beta}$, Assumption \ref{assum:critic_positive_definite} implies that $\xi^\top A_{v}(\theta)\xi \geq (\lambda/2)\|\xi\|^2$ for all $\xi =[\eta,\zeta]^\top$ such that $\zeta \in \ker(M_\theta)^{\perp}$, for all $\theta \in \Theta$.
\end{lemma}
\begin{proof}[Proof of Lemma \ref{lem:td-pd}]
     Recall that $A_{v}(\theta)=\E_{\theta}[A_{v}(z)]$ where $\E_{\theta}$ denotes expectation over the distribution of $z=(s, a, s')$ where $(s, a)\sim \nu^{\pi_\theta}$, $s'\sim P(\cdot|s, a)$. Hence, for any $\xi=[\eta, \zeta]$ such that $\zeta \in \ker(M_\theta)^{\perp}$, we have
    \begin{align}
    \begin{split}
    \xi^{\top}A_{v}(\theta)\xi &= c_{\beta}\eta^2 + \eta\zeta^{\top}\E_{\theta}\left[\phi(s)\right] + \zeta^{\top}\E_{\theta}\left[\phi(s)\left[\phi(s)-\phi(s')\right]^{\top}\right]\zeta\\
    &\overset{(a)}{\geq} c_{\beta}\eta^2 - |\eta|\norm{\zeta} + \lambda \norm{\zeta}^2\\
    &\geq \norm{\xi}^2\left\{\min_{u\in[0, 1]} c_{\beta} u - \sqrt{u(1-u)}+\lambda(1-u)\right\}\overset{(b)}{\geq}(\lambda/2)\norm{\xi}^2
    \end{split}
    \end{align}
    where $(a)$ is a consequence of Assumption \ref{assum:critic_positive_definite} and the fact that $\norm{\phi(s)}\leq 1$, $\forall s\in\mathcal{S}$. Finally, $(b)$ is satisfied when $c_{\beta} \geq \lambda + \sqrt{\frac{1}{\lambda^2}-1}$. This concludes the proof of Lemma \ref{lem:td-pd}.
\end{proof}

Given event $\mathcal{E}_B$, it follows that all states visited after the critic update, $s^{kh}_b$, where $h\geq 1$, $s^{kh}_b \in \mathcal{S}_R^{\theta_k}$. For $s\in \mathcal{S}_R^{\theta_k}$, we have the following result.
\begin{lemma}
\label{lem:kernel_recurrent}
Let $s \in \mathcal{S}_R^\theta$ and $ s' \sim P^{\pi_{\theta}}(\cdot \mid s) $ is the next state under the transition kernel induced by policy $ \pi_{\theta} $. Then $\ker(A_v(\theta)) \subset \ker(A_v(\theta,z))$, where $z=(s,a,s')$.
\end{lemma}
\begin{proof}
Let $z \in \ker(A_v(\theta))$. Then $z = [0,\, \zeta^\top]^\top$, where $\zeta \in \ker(M_\theta) = Z_\theta$. It follows that

$$
A_v(\theta) z = \begin{bmatrix} 0 \\ \phi(s)(\phi(s) - \phi(s'))^\top \zeta \end{bmatrix}.
$$

By the definition of $Z_\theta$, we have

$$
\phi(s)(\phi(s) - \phi(s'))^\top \zeta = \phi(s)(\phi(s)^\top \zeta - \phi(s')^\top \zeta) = 0,
$$

since $s, s' \in \mathcal{S}_R^{\theta}$.

\end{proof}

In particular, $(A_8)$ and $(A_9)$ for the critic update with 
$h\geq 1$ follow from Lemmas \ref{lem:td-pd} and \ref{lem:kernel_recurrent}, respectively. For any $z=(s, a, s')\in\mathcal{S}\times\mathcal{A}\times\mathcal{S}$, we have the following.
\begin{align}
\label{eq_appndx_washim_54}
    &\norm{A_{v}(\theta,z)}\leq |c_{\beta}|+\norm{\phi(s)} + \norm{\phi(s)(\phi(s)-\phi(s'))^{\top}}\overset{(a)}{\leq} c_{\beta}+3=\mathcal{O}(c_{\beta}),\\
    \label{eq_appndx_washim_55}
    &\norm{b_{v}(\theta,z)} \leq |c_{\beta}r(s, a)| + \norm{r(s, a)\phi(s)} \overset{(b)}{\leq} c_{\beta}+1 =\mathcal{O}(c_{\beta})
\end{align}
where $(a)$, $(b)$ hold since $|r(s, a)|\leq 1$ and $\norm{{\phi(s)}}\leq 1$, $\forall (s, a)\in\mathcal{S}\times\mathcal{A}$. This verifies $(A_6)$ and $(A_7)$. 
Using Lemmas \ref{lem:bias-markov} and \ref{lem:variance-markov} with the above bounds, we obtain
\begin{align}
    \norm{\E \left[\frac{1}{B}\sum_{i=1}^B A_v(\theta,z_i)\right]-A_v(\theta)}  \leq \cO\left(\frac{\sqrt{m}c_{\beta} C_{\mathrm{tar}}}{B}\right)
\end{align}
and
\begin{align}
     \E\norm{\left[\frac{1}{B}\sum_{i=1}^BA_v(\theta,z_i)\right]-A_v(\theta)}^2 \leq \cO\left( \frac{c_\beta^2 + \sqrt{m}c_\beta^2C_{\mathrm{tar}} }{B}\right).
\end{align}
Similarly,
\begin{align}
    \norm{\E \left[\frac{1}{B}\sum_{i=1}^B b_v(\theta,z_i)\right]-b_v(\theta)}  \leq \cO\left(\frac{\sqrt{m}c_{\beta} C_{\mathrm{tar}}}{B}\right)
\end{align}
and
\begin{align}
     \E\norm{\left[\frac{1}{B}\sum_{i=1}^Bb_v(\theta,z_i)\right]-b_v(\theta)}^2 \leq \cO\left( \frac{c_\beta^2 + \sqrt{m}c_\beta^2C_{\mathrm{tar}} }{B}\right).
\end{align}

These yield conditions $(A_1)$–$(A_4)$ for the critic update, and condition $(A_5)$ follows by setting $\bar{\delta}_q = \delta_q$. Since all conditions $(A_1)$–$(A_9)$ are now verified, we can apply Theorem~\ref{thm_2} with the critic learning rates $c_{\beta} = \lambda + \sqrt{\frac{1}{\lambda^2} - 1}$ and $\beta = \frac{\lambda^2}{2}$ to obtain the following bounds.

\begin{align*}
\E\left[\|\Pi(\xi_H - \xi^*)\|^2\right] \leq &\cO\Bigg(\exp\left(-\tfrac{H \lambda^3}{16}\right)\|\xi_1-\xi^*\|^2 +  \frac{C_{\mathrm{tar}}\sqrt{m}}{\lambda^{4}B}  + \frac{C_{\mathrm{tar}}\sqrt{m}}{\lambda B} + \frac{C_{\mathrm{tar}}^2 m}{\lambda^6 B^2}+\frac{C_{\mathrm{tar}}^2 m}{\lambda^2 B^2}\Bigg).
\end{align*}
and
\begin{align*}
\|\Pi(\E[x_H] - x^*)\|^2 \leq \cO\Bigg(\exp\left(-\tfrac{H \lambda^3}{16}\right) \|\xi_1-\xi^*\|^2 + \frac{C_{\mathrm{tar}}^2 m\|\xi_1-\xi^*\|^2}{\lambda^2 B^2}+\frac{C_{\mathrm{tar}}^2 m}{\lambda^6 B^2}+\frac{C_{\mathrm{tar}}^2 m}{\lambda^2 B^2}\Bigg).
\end{align*}

\section{Proof of Theorem \ref{thm:NPG-final}}
\label{app:npg-final}
We note that condition $(A_9)$ holds trivially for the NPG update, since $\ker(A_u(\theta)) = \{0\}$. Furthermore, by Assumption~\ref{assump:FND_policy}, condition $(A_8)$ is readily satisfied with $\lambda_P = \mu$. Finally, from Assumption~\ref{assump:score_func_bounds}, we have
\begin{align}
   \|A_u(\theta,z)\|= \|\nabla \log \pi_{\theta}(a|s) \nabla \log \pi_{\theta}(a|s)^\top\|\leq \|\nabla \log \pi_{\theta}(a|s) \|^2 \leq G_1^2,
\end{align}
for all $z=(s,a,s')$. Combining this bound with Lemmas \ref{lem:bias-markov} and \ref{lem:variance-markov}, it follows that
\begin{align}
    \norm{\E \left[\frac{1}{B}\sum_{i=1}^B A_u(\theta,z_i)\right]-A_u(\theta)}  \leq \cO\left(\frac{\sqrt{d}G_1^2 C_{\mathrm{tar}}}{B}\right)
\end{align}
and
\begin{align}
     \E\norm{\left[\frac{1}{B}\sum_{i=1}^BA_u(\theta,z_i)\right]-A_u(\theta)}^2 \leq \cO\left( \frac{G_1^4 + \sqrt{d}G_1^4C_{\mathrm{tar}} }{B}\right).
\end{align}

The above bounds yield $(A_6)$, $(A_2)$ and $(A_1)$, respectively. To obtain $(A_7)$, we combine the bound on the Advantage function from Lemma~\ref{lem:q-v-a-bound} with Assumption~\ref{assump:score_func_bounds} to bound the policy gradient as follows
\begin{align}
\begin{split}
    \|\nabla J(\theta)\|&=\|\E[A^{\pi_\theta}(s,a)\nabla\log \pi_\theta(a|s)]\|\leq \|A^{\pi_\theta}(s,a)\|\|\nabla\log \pi_\theta(a|s)\|\\
    &\leq (1+4(C_{\mathrm{hit}}+C_{\mathrm{tar}}))G_1.
\end{split}
\end{align}

To prove the other statements, recall the definition of $b_u(\theta_k, \xi_k, \cdot)$ from \eqref{eq:ut_exp}. Let $\E_{u,z}$ denote the expectation over $\{z_i\}_{i=1}^B$, given the entire history prior to $z_1$ (including $\xi_k$), $\E_{u}$ be denote the expectation given the entire history prior to $z_1$ with $s_1\sim d^{\pi_{\theta_k}}$ and $\E_{v}$ denote the expectation over the entire history prior to $z_1$. Observe the following relations for arbitrary $\theta_k, \xi_k$.
\begin{align*}
    &\E_{u}\left[\frac{1}{B}\sum_{i=1}^B b_u(\theta_k, \xi_k, z_i)\right] - \nabla_\theta J(\theta_k)\\
    &= \E_{u}\left[\frac{1}{B}\sum_{i=1}^B\bigg\{r(s_i, a_i) - \eta_k + \langle \phi(s_{i+1}) - \phi(s_i), \zeta_k \rangle\bigg\}\nabla_{\theta} \log_{\pi_{\theta_k}}(a_i|s_i) \right] - \nabla_\theta J(\theta_k)\\
    &\overset{(a)}{=}\underbrace{\E_{u}\left[\frac{1}{B}\sum_{i=1}^B\bigg\{\eta_k^* - \eta_k + \langle \phi(s_{i+1}) - \phi(s_i), \zeta_k-\zeta_k^* \rangle\bigg\}\nabla_{\theta} \log_{\pi_{\theta_k}}(a_i|s_i) \right]}_{T_0}+\\
    &\hspace{0.75cm}+ \underbrace{\E_{u}\left[\frac{1}{B}\sum_{i=1}^B\bigg\{ \left(\langle \phi(s_i), \zeta_k^*\rangle -V^{\pi_{\theta_k}}(s_i)\right) + \left(V^{\pi_{\theta_k}}(s_{i+1})-\langle \phi(s_{i+1}), \zeta_k^*\rangle \right)\bigg\}\nabla_{\theta} \log_{\pi_{\theta_k}}(a_i|s_i) \right]}_{T_1}\\
    &\hspace{0.75cm}+ \underbrace{\E_{u}\left[\frac{1}{B}\sum_{i=1}^B\bigg\{r(s_i, a_i)- \eta_k^*  +V^{\pi_{\theta_k}}(s_{i}) - V^{\pi_{\theta_k}}(s_{i+1})\bigg\}\nabla_{\theta} \log_{\pi_{\theta_k}}(a_i|s_i) \right]-\nabla_{\theta} J(\theta_k)}_{T_2}
\end{align*}
We have $T_2=0$ as a consequence of the Bellman's equation \eqref{eq_bellman}. Whereas,
\begin{align*}
    \|T_1\|^2\leq 2\E_{d^{\pi_{\theta_k}}}|\langle \phi(s_i), \zeta_k^*\rangle -V^{\pi_{\theta_k}}(s_i)|^2 G_1^2\leq 4 G_1^2 \epsilon_{\mathrm{app}}.
\end{align*}
Finally, note that
\begin{align*}
    T_0&=\E_{u}\left[\frac{1}{B}\sum_{i=1}^B\bigg\{\eta_k^* - \eta_k + \langle \phi(s_{i+1}) - \phi(s_i), \zeta_k-\zeta_k^* \rangle\bigg\}\nabla_{\theta} \log_{\pi_{\theta_k}}(a_i|s_i) \right]\\
    &=\E_{u}\left[\frac{1}{B}\sum_{i=1}^B\bigg\{\eta_k^* - \eta_k + \langle \phi(s_{i+1}) - \phi(s_i), \Pi (\zeta_k-\zeta_k^*) \rangle\bigg\}\nabla_{\theta} \log_{\pi_{\theta_k}}(a_i|s_i) \right]
\end{align*}
which yields
\begin{align}
    \|T_0\|^2&\leq G_1^2 \left(\|\eta_k-\eta_k^*\|^2+\|\Pi (\zeta_k-\zeta_k^*) \|^2\right)=G_1^2\|\Pi(\xi_k-\xi_k^*)\|^2.
\end{align}
Moreover, from Lemmas~\ref{lem:bias-markov} and \ref{lem:variance-markov}, we have
\begin{align}
    \norm{\E_{u,z} \left[\frac{1}{B}\sum_{i=1}^B b_u(\theta_k,\xi_k,z_i)\right]-\E_{u} \left[\frac{1}{B}\sum_{i=1}^B b_u(\theta_k,\xi_k,z_i)\right]}  \leq \cO\left(\frac{\sqrt{d}G_1^2\|\Pi \xi_k\|^2 C_{\mathrm{tar}}}{B}\right)
\end{align}
and
\begin{align}
     \E_{u,z}\norm{\left[\frac{1}{B}\sum_{i=1}^Bb_u(\theta_k,\xi_k,z_i)\right]-\E_{u} \left[\frac{1}{B}\sum_{i=1}^B b_u(\theta_k,\xi_k,z_i)\right]}^2 \leq \cO\left( \frac{\sqrt{d}G_1^2\|\Pi \xi_k\|^2C_{\mathrm{tar}} }{B}\right).
\end{align}
It follows that
\begin{align}
    &\E\norm{\frac{1}{B}\sum_{i=1}^Bb_u(\theta_k,\xi_k,z_i)-\nabla_{\theta} J(\theta_k)}^2 \nonumber\\
    &\quad \quad\leq \cO\left(\frac{\sqrt{d}G_1^2\E\|\Pi \xi_k\|^2 C_{\mathrm{tar}}}{B}+G_1^2\E\|\Pi(\xi_k-\xi_k^*)\|^2+G_1^2 \epsilon_{\mathrm{app}}\right)
\end{align}
and
\begin{align}
    &\norm{\E_{u,z}\left[\frac{1}{B}\sum_{i=1}^Bb_u(\theta_k,\xi_k,z_i)\right]-\nabla_{\theta} J(\theta_k)}^2 \nonumber \\
    &\quad \quad \leq \cO\left(\frac{dG_1^4\E\|\Pi \xi_k\|^2 C_{\mathrm{tar}}^2}{B^2}+G_1^2\E\|\Pi(\xi_k-\xi_k^*)\|^2+G_1^2 \epsilon_{\mathrm{app}}\right)
\end{align}
Conditions $(A_3)$ and $(A_4)$ now follow. In contrast to the critic analysis, we employ a sharper bound for $(A_5)$, which is necessary to obtain order-optimal regret. We derive this below.
\begin{align}
\begin{split}
       &\E\left[\frac{1}{B}\sum_{i=1}^B b_u(\theta_k, \xi_k, z_i)\right] - \nabla_\theta J(\theta_k)\\
       &=\E_{u}\left[\E_v\left[\frac{1}{B}\sum_{i=1}^B b_u(\theta_k, \xi_k, z_i)\right]\right] - \nabla_\theta J(\theta_k)\\
    &\overset{(a)}{=}\underbrace{\E_{u}\left[\frac{1}{B}\sum_{i=1}^B\bigg\{\eta_k^* - \E_v[\eta_k] + \langle \phi(s_{i+1}) - \phi(s_i), \E_v[\zeta_k]-\zeta_k^* \rangle\bigg\}\nabla_{\theta} \log_{\pi_{\theta_k}}(a_i|s_i) \right]}_{T_0}+\\
    &+ \underbrace{\E_{u}\left[\frac{1}{B}\sum_{i=1}^B\bigg\{ \left(\langle \phi(s_i), \zeta_k^*\rangle -V^{\pi_{\theta_k}}(s_i)\right) + \left(V^{\pi_{\theta_k}}(s_{i+1})-\langle \phi(s_{i+1}), \zeta_k^*\rangle \right)\bigg\}\nabla_{\theta} \log_{\pi_{\theta_k}}(a_i|s_i) \right]}_{T_1}\\
    &+ \underbrace{\E_{u}\left[\frac{1}{B}\sum_{i=1}^B\bigg\{r(s_i, a_i)- \eta_k^*  +V^{\pi_{\theta_k}}(s_{i}) - V^{\pi_{\theta_k}}(s_{i+1})\bigg\}\nabla_{\theta} \log_{\pi_{\theta_k}}(a_i|s_i) \right]-\nabla_{\theta} J(\theta_k)}_{T_2},
\end{split}
\end{align}
With the above decomposition, we obtain
\begin{align}
    &\norm{\E\left[\frac{1}{B}\sum_{i=1}^Bb_u(\theta_k,\xi_k,z_i)\right]-\nabla_{\theta} J(\theta_k)}^2 \nonumber \\
    &\quad \quad \leq \cO\left(\frac{dG_1^4\E\|\Pi \xi_k\|^2 C_{\mathrm{tar}}^2}{B^2}+G_1^2\|\Pi(\E[\xi_k]-\xi_k^*)\|^2+G_1^2 \epsilon_{\mathrm{app}}\right)
\end{align}
This concludes the verification of condition $(A_5)$. Note that the bound involves the term $\|\Pi(\mathbb{E}[\xi_k] - \xi_k^*)\|^2$, rather than $\mathbb{E}[\|\Pi(\xi_k - \xi_k^*)\|^2]$, making it significantly sharper. We can now invoke Theorem~\ref{thm_2} by setting the NPG step-size $\gamma \coloneqq \frac{\mu}{G_1^2}$ to derive bounds on both the second-order error and the bias of the NPG estimate $\omega_k$, as follows.
\begin{align}
\label{eq:NPG-Var_const}
\begin{split}
&\E\left[\|\omega_k - \omega_k^*\|^2\right]\\
&\leq \cO\Bigg(\exp\left(-\tfrac{H \mu^2}{4(1+4C)G_1}\right)\|\omega_0-\omega_k^*\|^2 +  \frac{\sqrt{d}G_1^6 C_{\mathrm{tar}}}{\mu^2B}+\E\|\Pi(\xi_k-\xi_k^*)\|^2+ \epsilon_{\mathrm{app}} \\
&\quad+ \frac{\sqrt{d}G_1^4C_{\mathrm{tar}}C^2}{\mu^2 B}+ \frac{dG_1^4 C^2 C_{\mathrm{tar}}^2}{\mu^4 B^2}+\frac{dG_1^8 C_{\mathrm{tar}}^2}{\mu^2 B^2}+G_1^2\mu^{-2}\E\|\Pi(\xi_k-\xi_k^*)\|^2+G_1^2 \mu^{-2}\epsilon_{\mathrm{app}}\Bigg)\\
&\quad \overset{(a)}{\leq} \cO\Bigg(\exp\left(-\tfrac{H \mu^2}{4(1+4C)G_1}\right)\|\omega_0-\omega_k^*\|^2 +  \frac{\sqrt{d}G_1^6 C_{\mathrm{tar}}}{\mu^2B}+ G_1^2 \mu^{-2}\exp\left(-\tfrac{H \lambda^3}{16}\right)\|\xi_1-\xi^*\|^2\\
&\quad +  G_1^2 \mu^{-2}\frac{C_{\mathrm{tar}}\sqrt{m}}{\lambda^{4}B}  + G_1^2 \mu^{-2} \frac{C_{\mathrm{tar}}\sqrt{m}}{\lambda B} + G_1^2 \mu^{-2}\frac{C_{\mathrm{tar}}^2 m}{\lambda^6 B^2}+G_1^2 \mu^{-2}\frac{C_{\mathrm{tar}}^2 m}{\lambda^2 B^2}+ G_1^2 \mu^{-2}\epsilon_{\mathrm{app}} \\
&\quad + \frac{\sqrt{d}G_1^4C_{\mathrm{tar}}C^2}{\mu^2 B}+ \frac{dG_1^4 C^2 C_{\mathrm{tar}}^2}{\mu^4 B^2}+\frac{dG_1^8 C_{\mathrm{tar}}^2}{\mu^2 B^2}+G_1^2 \mu^{-2}\epsilon_{\mathrm{app}}\Bigg),
\end{split}
\end{align}
where $(a)$ follows using the second-order bound in Theorem \ref{thm:critic-final} and
\begin{align}
\label{eq:NPG-bias_const}
\begin{split}
\|\E[\omega_k] - \omega_k^*\|^2 &\leq \cO\Bigg(\exp\left(-\tfrac{H \mu^2}{4(1+4C)G_1}\right)\|\omega_0-\omega_k^*\|^2+ \frac{dG_1^8 C_{\mathrm{tar}}^2}{\mu^4 B^2}+\frac{dG_1^8 C_{\mathrm{tar}}^2}{\mu^6 B^2}+\frac{dG_1^6 C_{\mathrm{tar}}^2}{\mu^4 B^2}\\
& +G_1^2\mu^{-2}\|\Pi(\E[\xi_k]-\xi_k^*)\|^2+G_1^2 \mu^{-2}\epsilon_{\mathrm{app}}\Bigg)\\
& \overset{(b)}{\leq} \cO\Bigg(\exp\left(-\tfrac{H \mu^2}{4(1+4C)G_1}\right)\|\omega_0-\omega_k^*\|^2+ \frac{dG_1^8 C_{\mathrm{tar}}^2}{\mu^4 B^2}+\frac{dG_1^8 C_{\mathrm{tar}}^2}{\mu^6 B^2}+\frac{dG_1^6 C_{\mathrm{tar}}^2}{\mu^4 B^2}\\
& +G_1^2\mu^{-2}\exp\left(-\tfrac{H \lambda^3}{16}\right) \|\xi_1-\xi^*\|^2 + G_1^2\mu^{-2}\frac{C_{\mathrm{tar}}^2 m\|\xi_1-\xi^*\|^2}{\lambda^2 B^2}\\
& +G_1^2\mu^{-2}\frac{C_{\mathrm{tar}}^2 m}{\lambda^6 B^2}+G_1^2\mu^{-2}\frac{C_{\mathrm{tar}}^2 m}{\lambda^2 B^2}+G_1^2 \mu^{-2}\epsilon_{\mathrm{app}}\Bigg),
\end{split}
\end{align}
where $(b)$ follows using the bias bound in Theorem \ref{thm:critic-final} and

\section{Proof of Lemma \ref{lem:regret-decomp}}
\label{app:proof-lemma1}

The regret can be decomposed as follows, following the standard approach used in \citep{wei2020model,bai2023regret}:
\begin{align}
\begin{split}
    \mathrm{Reg}_T &= \sum_{t=0}^{T-1} \left(J^* - r(s_t, a_t)\right) \\
    &= HB\sum_{k=1}^{K} \left(J^* - J(\theta_k)\right) + \sum_{k=1}^{K} \sum_{t \in \mathcal{I}_k} \left(J(\theta_k) - r(s_t, a_t)\right) \\
    &= HB\sum_{k=1}^{K} \left(J^* - J(\theta_k)\right) + \E \left[\sum_{k=1}^{K} V^{\pi_{\theta_{k+1}}}(s_{kH}) - V^{\pi_{\theta_k}}(s_{kH}) \right] + \E \left[ V^{\pi_{\theta_K}}(s_T) - V^{\pi_{\theta_0}}(s_0) \right].
\end{split}
\end{align}

Since $0 \leq V^{\pi}(s) \leq 2C$ for all $\pi \in \Pi$ and $s \in \mathcal{S}$ from Lemma \ref{lem:q-v-a-bound}, it follows that
\begin{align}
\label{eq:19}
\E[\mathrm{Reg}_T] \leq HB\sum_{k=1}^{K} \left(J^* - \E[J(\theta_k)]\right) + 2C(K+1).
\end{align}
We now decompose the term $\sum_{k=1}^{K} (J^* - \E[J(\theta_k)])$. There are several existing results for this. We use Lemma 1 from \citep{ganesh2024order}, re-stated below, since it is sharper than most existing results.
\begin{lemma}
    \label{lemma:local_global}
    Consider any policy update rule of form
    \begin{align}
        \theta_{k+1} = \theta_k + \alpha \omega_k.
    \end{align}
     If \ref{assump:score_func_bounds} holds, then the following inequality is satisfied for any $K$.
     \begin{align}
			J^{*}-\frac{1}{K}\sum_{k=0}^{K-1}\E[J(\theta_k)]\leq &\sqrt{\epsilon_{\mathrm{bias}}}+\frac{G_1}{K}\sum_{k=0}^{K-1}\E\Vert(\E_k\left[\omega_k\right]-\omega^*_k)\Vert +\frac{\alpha G_2}{2K}\sum_{k=0}^{K-1}\E\Vert \omega_k\Vert^2\nonumber \\
            &+\frac{1}{\alpha K}\E_{s\sim d^{\pi^*}}[\mathrm{KL}(\pi^*(\cdot\vert s)\Vert\pi_{\theta_0}(\cdot\vert s))]
		\end{align}
  where $\mathrm{KL}(\cdot \|\cdot)$ is the Kullback-Leibler divergence, $\omega^*_k$ is the NPG direction $F(\theta_k)^{-1}\nabla J(\theta_k)$, $\pi^*$ is the optimal policy, $J^*$ is the optimal value of the function $J(\cdot)$, and $\E_k[\cdot]$ denotes conditional expectation given the history up to epoch $k$.
\end{lemma}
Using the above result, we obtain the following.
   \begin{align}
   \label{eq:20}
   \begin{split}
			HB\sum_{k=0}^{K-1}(J^{*}-\E[J(\theta_k)])\leq &T\sqrt{\epsilon_{\mathrm{bias}}}+\frac{BH G_1}{K}\sum_{k=0}^{K-1}\E\Vert(\E_k\left[\omega_k\right]-\omega^*_k)\Vert +\frac{\alpha G_2BH}{K}\sum_{k=0}^{K-1}\E\Vert \omega_k\Vert^2 \\
            &+\frac{BH}{\alpha }\E_{s\sim d^{\pi^*}}[\mathrm{KL}(\pi^*(\cdot\vert s)\Vert\pi_{\theta_0}(\cdot\vert s))]
            % \leq &T\sqrt{\epsilon_{\mathrm{bias}}}+T G_1\E\Vert(\E_k\left[\omega_k\right]-\omega^*_k)\Vert +\frac{\mu^2G_2T}{4G_1^2L}\E\Vert \omega_k\Vert^2\nonumber \\
            % &+\frac{4G_1^2L \sqrt{T}\log T}{\mu^2} \E_{s\sim d^{\pi^*}}[\mathrm{KL}(\pi^*(\cdot\vert s)\Vert\pi_{\theta_0}(\cdot\vert s))]
        \end{split}
		\end{align}

 The term containing $\E\|\omega_k\|^2$ can be further decomposed as
\begin{align}
\label{eq:eq_21}
    \begin{split}
        \frac{\alpha G_2BH}{K}\E\Vert \omega_k \Vert^2&\leq \frac{2\alpha G_2BH}{K}\E\Vert \omega_k -\omega_k^*\Vert^2 + \frac{2\alpha G_2BH}{K}\E\Vert \omega_k^* \Vert^2\\
        &\overset{(a)}{\leq}\frac{2\alpha G_2BH}{K}\E\Vert \omega_k -\omega_k^*\Vert^2 + \frac{2\alpha G_2BH}{\mu^2 K}\E\Vert \nabla_{\theta} J(\theta_k) \Vert^2
    \end{split}
\end{align}
% \begin{align}
% \label{eq:eq_21}
%     \begin{split}
%         \frac{\mu^2G_2T}{4G_1^2L}\E\Vert \omega_k \Vert^2&\leq \frac{\mu^2G_2T}{2G_1^2L}\E\Vert \omega_k -\omega_k^*\Vert^2 + \frac{\mu^2G_2T}{2G_1^2L}\E\Vert \omega_k^* \Vert^2\\
%         &\overset{(a)}{\leq} \frac{\mu^2G_2T}{2G_1^2L}\E\Vert \omega_k -\omega_k^*\Vert^2 + \frac{G_2T}{2G_1^2L}\E\Vert \nabla_{\theta} J(\theta_k) \Vert^2
%     \end{split}
% \end{align}
where $(a)$ follows from Assumption \ref{assump:FND_policy} and the definition that $\omega_k^*=F(\theta_k)^{-1}\nabla_\theta J(\theta_k)$. The result now follows by substituting \eqref{eq:20} and \eqref{eq:eq_21} in \eqref{eq:19}.

  %  \begin{align}
		% 	&HB\sum_{k=0}^{K-1}(J^{*}-\E[J(\theta_k)])\\
  %           &\leq T\sqrt{\epsilon_{\mathrm{bias}}}+T G_1\E\Vert(\E_k\left[\omega_k\right]-\omega^*_k)\Vert +\frac{\mu^2G_2T}{4G_1^2L}\E\Vert \omega_k\Vert^2\nonumber \\
  %           &+\frac{4G_1^2L \sqrt{T}\log T}{\mu^2} \E_{s\sim d^{\pi^*}}[\mathrm{KL}(\pi^*(\cdot\vert s)\Vert\pi_{\theta_0}(\cdot\vert s))]\\
  %           & \leq T\sqrt{\epsilon_{\mathrm{bias}}}+T G_1\E\Vert(\E_k\left[\omega_k\right]-\omega^*_k)\Vert + \frac{\mu^2G_2T}{2G_1^2L}\E\Vert \omega_k -\omega_k^*\Vert^2 + \frac{G_2T}{2G_1^2L}\E\Vert \nabla_{\theta} J(\theta_k) \Vert^2\nonumber \\
  %           &+\frac{4G_1^2L \sqrt{T}\log T}{\mu^2} \E_{s\sim d^{\pi^*}}[\mathrm{KL}(\pi^*(\cdot\vert s)\Vert\pi_{\theta_0}(\cdot\vert s))]
		% \end{align}

% Recall that the global convergence of any update of form $\theta_{k+1}=\theta_{k}+\alpha \omega_k$ can be bounded as
% \begin{equation}
% \label{eq_appndx_wash_59}
% 			\begin{split}
% 			J^{*}-\frac{1}{K}\sum_{k=0}^{K-1}&\E[J(\theta_k)]\leq \sqrt{\epsilon_{\mathrm{bias}}} +\frac{G_1}{K}\sum_{k=0}^{K-1}\E\Vert(\E_k\left[\omega_k\right]-\omega^*_k)\Vert+\dfrac{\alpha G_2}{K}\sum_{k=0}^{K-1}\E\Vert \omega_k -\omega_k^*\Vert^2\\
%             & + \dfrac{\alpha \mu^{-2}}{ K}\sum_{k=0}^{K-1}\E\Vert \nabla_{\theta} J(\theta_k) \Vert^2
%             +\frac{1}{\alpha K}\E_{s\sim d^{\pi^*}}[\mathrm{KL}(\pi^*(\cdot\vert s)\Vert\pi_{\theta_0}(\cdot\vert s))].		\end{split}
% 		\end{equation}
\section{Proof of Theorem \ref{thm:main}}
\label{app:final}
  We shall now derive a bound for \(\textstyle{\sum_{k=0}^{K-1}\Vert\nabla_\theta J(\theta_k)\Vert^2}\). Given that the function \(J\) is \(L\)-smooth, we obtain:
\begin{align}
\label{eq:eq_26}
\begin{split}
    &J(\theta_{k+1})\\
    &\geq J(\theta_k)+\left<\nabla_\theta J(\theta_k),\theta_{k+1}-\theta_k\right>-\frac{L}{2}\Vert\theta_{k+1}-\theta_k\Vert^2\\
    &=J(\theta_k)+\alpha\left<\nabla_\theta J(\theta_k),\omega_k\right>-\frac{\alpha^2 L}{2}\Vert \omega_k\Vert^2\\
    &=J(\theta_k)+\alpha\left<\nabla_\theta J(\theta_k),\omega_k^*\right>+\alpha\left<\nabla_\theta J(\theta_k),\omega_k-\omega_k^*\right>-\frac{\alpha^2 L}{2}\Vert \omega_k-\omega_k^*+\omega_k^*\Vert^2\\
    &\overset{(a)}{\geq} J(\theta_k) +\alpha\left<\nabla_\theta J(\theta_k),F(\theta_k)^{-1}\nabla_\theta J(\theta_k)\right> +\alpha\left<\nabla_\theta J(\theta_k),\omega_k-\omega_k^*\right> \\
    &\hspace{1cm}-\alpha^2 L \Vert \omega_k-\omega_k^*\Vert^2 - \alpha^2L\Vert \omega_k^*\Vert^2\\
    &\overset{(b)}{\geq} J(\theta_k) +\dfrac{\alpha}{G_1^2}\Vert\nabla_\theta J(\theta_k)\Vert^2 +\alpha\left<\nabla_\theta J(\theta_k),\omega_k-\omega_k^*\right> -\alpha^2 L \Vert \omega_k-\omega_k^*\Vert^2 - \alpha^2L\Vert \omega_k^*\Vert^2\\
    &=J(\theta_k) +\dfrac{\alpha}{2G_1^2}\Vert\nabla_\theta J(\theta_k)\Vert^2 + \dfrac{\alpha}{2G_1^2}\Vert\nabla_\theta J(\theta_k)+G_1^2(\omega_k-\omega_k^*)\Vert^2-\left(\dfrac{\alpha G_1^2}{2}+\alpha^2 L\right) \Vert \omega_k-\omega_k^*\Vert^2 \\
    &\hspace{1cm}- \alpha^2L\Vert \omega_k^*\Vert^2 \\
    &\geq J(\theta_k) +\dfrac{\alpha}{2G_1^2}\Vert\nabla_\theta J(\theta_k)\Vert^2 -\left(\dfrac{\alpha G_1^2}{2}+\alpha^2 L\right) \Vert \omega_k-\omega_k^*\Vert^2 - \alpha^2L\Vert F(\theta_k)^{-1}\nabla_\theta J(\theta_k)\Vert^2 \\
    &\overset{(c)}{\geq} J(\theta_k) +\left(\dfrac{\alpha}{2G_1^2}-\dfrac{\alpha^2 L}{\mu^2}\right)\Vert\nabla_\theta J(\theta_k)\Vert^2 -\left(\dfrac{\alpha G_1^2}{2}+\alpha^2 L\right) \Vert \omega_k-\omega_k^*\Vert^2  
\end{split}
\end{align}
 where $(a)$ use the Cauchy-Schwarz inequality and the definition that $\omega_k^*=F(\theta_k)^{-1}\nabla_\theta J(\theta_k)$. Relations $(b)$, and $(c)$ are consequences of Assumption \ref{assump:score_func_bounds}(a) and \ref{assump:FND_policy} respectively. Summing the above inequality over $k\in\{0,\cdots, K-1\}$, rearranging the terms and substituting $\alpha = \frac{\mu^2}{4G_1^2L}$, we obtain
 \begin{align}
 \label{eq:local-temp}
 \begin{split}
     \dfrac{\mu^2}{16G_1^4 L}\left(\sum_{k=0}^{K-1}\Vert\nabla_\theta J(\theta_k)\Vert^2\right)&\leq J(\theta_K)-J(\theta_0) + \left(\dfrac{\mu^2}{8L}+\dfrac{\mu^4}{16G_1^4 L}\right)\left(\sum_{k=0}^{K-1}\Vert \omega_k-\omega_k^*\Vert^2\right)\\
     &\overset{(a)}{\leq}2+\left(\dfrac{\mu^2}{8L}+\dfrac{\mu^4}{16G_1^4 L}\right)\left(\sum_{k=0}^{K-1}\Vert \omega_k-\omega_k^*\Vert^2\right)
 \end{split}
 \end{align}
where $(a)$ uses the fact that $J(\cdot)$ is absolutely bounded above by $1$. Inequality \eqref{eq:local-temp} can be simplified as follows.
\begin{align}
\label{eq_appndx_wash_62}
 \begin{split}
     &\left(\sum_{k=0}^{K-1}\E\Vert\nabla_\theta J(\theta_k)\Vert^2\right)\\
     &\hspace{1cm}\leq  \dfrac{32LG_1^4\mu^2}{\mu^{4} }+\left(2G_1^4+\mu^2\right)\left(\sum_{k=0}^{K-1}\E \Vert \omega_k-\omega_k^*\Vert^2\right)\\
     &\hspace{1cm}\leq \dfrac{32LG_1^4\mu^2}{\mu^{4} }+\left(2G_1^4+\mu^2\right)K\Bigg(   \exp\left(-\tfrac{H \mu^2}{4(1+4C)G_1}\right)\|\omega_0-\omega_k^*\|^2 \\
     &\hspace{1cm}+   G_1^2 \mu^{-2}\exp\left(-\tfrac{H \lambda^3}{16}\right)\|\xi_1-\xi^*\|^2 +  G_1^2 \mu^{-2}\frac{C_{\mathrm{tar}}\sqrt{m}}{\lambda^{4}B}   + G_1^2 \mu^{-2}\frac{C_{\mathrm{tar}}^2 m}{\lambda^6 B^2}\\
     &\hspace{1cm}+  \frac{\sqrt{d}G_1^4C_{\mathrm{tar}}C^2}{\mu^2 B}+ \frac{dG_1^4 C^2 C_{\mathrm{tar}}^2}{\mu^4 B^2}+\frac{dG_1^8 C_{\mathrm{tar}}^2}{\mu^2 B^2}+G_1^2 \mu^{-2}\epsilon_{\mathrm{app}} \Bigg)\\.
 \end{split}
 \end{align}

Furthermore, using $\alpha = \frac{\mu^2}{4G_1^2L}$, we find that $HB\sum_{k=0}^{K-1}(J^{*}-\E[J(\theta_k)])$ can be written as follows.
   \begin{align}
   \label{eq:second-last}
   \begin{split}
			&\E[\mathrm{Reg}_T] \\
            &\leq T\sqrt{\epsilon_{\mathrm{bias}}}+T G_1\E\Vert(\E_k\left[\omega_k\right]-\omega^*_k)\Vert +\frac{\mu^2G_2T}{4G_1^2L}\E\Vert \omega_k\Vert^2\nonumber \\
            &+\frac{4G_1^2L \sqrt{T}\log T}{\mu^2} \E_{s\sim d^{\pi^*}}[\mathrm{KL}(\pi^*(\cdot\vert s)\Vert\pi_{\theta_0}(\cdot\vert s))]\\
            & \leq T\sqrt{\epsilon_{\mathrm{bias}}}+T G_1\E\Vert(\E_k\left[\omega_k\right]-\omega^*_k)\Vert + \frac{\mu^2G_2T}{2G_1^2L}\E\Vert \omega_k -\omega_k^*\Vert^2 + \frac{G_2T}{2G_1^2L}\E\Vert \nabla_{\theta} J(\theta_k) \Vert^2\nonumber \\
            &+\frac{4G_1^2L \sqrt{T}\log T}{\mu^2} \E_{s\sim d^{\pi^*}}[\mathrm{KL}(\pi^*(\cdot\vert s)\Vert\pi_{\theta_0}(\cdot\vert s))]+C(K+1).
        \end{split}
		\end{align}
Using the bound on $\mathbb{E}\|\nabla_{\theta} J(\theta_k)\|^2$ derived earlier, along with the second-order error and bias of the NPG estimates, including all constants, from \eqref{eq:NPG-Var_const} and \eqref{eq:NPG-bias_const}, we obtain the following regret bound, omitting all non-dominant terms.
% Recall that the global convergence of any update of form $\theta_{k+1}=\theta_{k}+\alpha \omega_k$ can be bounded as
% \begin{equation}
% \label{eq_appndx_wash_59}
% 			\begin{split}
% 			J^{*}-\frac{1}{K}\sum_{k=0}^{K-1}&\E[J(\theta_k)]\leq \sqrt{\epsilon_{\mathrm{bias}}} +\frac{G_1}{K}\sum_{k=0}^{K-1}\E\Vert(\E_k\left[\omega_k\right]-\omega^*_k)\Vert+\dfrac{\alpha G_2}{K}\sum_{k=0}^{K-1}\E\Vert \omega_k -\omega_k^*\Vert^2\\
%             & + \dfrac{\alpha \mu^{-2}}{ K}\sum_{k=0}^{K-1}\E\Vert \nabla_{\theta} J(\theta_k) \Vert^2
%             +\frac{1}{\alpha K}\E_{s\sim d^{\pi^*}}[\mathrm{KL}(\pi^*(\cdot\vert s)\Vert\pi_{\theta_0}(\cdot\vert s))].		\end{split}
% 		\end{equation}
\begin{align}
 \begin{split}
     &\E[\mathrm{Reg}_T]\\
     &\leq \cO\Bigg(G_2G_1^2\sqrt{T}(\log T)\bigg[\dfrac{\mu^2+L}{\mu^{2}L}\bigg]+\frac{G_2\sqrt{T}}{G_1^2L}\bigg[ \frac{G_1^2 C_{\mathrm{tar}}\sqrt{m}}{ \mu^{2}\lambda^{4}}   +  \frac{\sqrt{d}G_1^4C_{\mathrm{tar}}C^2}{\mu^2 }\bigg]\\
     &+\sqrt{T}\Bigg[\frac{\sqrt{d}G_1^5 C_{\mathrm{tar}}}{\mu^3 }+ \frac{C_{\mathrm{tar}} G_1^2\sqrt{m} R_0}{\mu \lambda }+\frac{G_1^2C_{\mathrm{tar}} \sqrt{m}}{\mu \lambda^3 }\Bigg]+\frac{TG_1^2\sqrt{\epsilon_{\mathrm{app}}}}{ \mu}+T \sqrt{\epsilon_{\mathrm{bias}}}\Bigg).
 \end{split}
 \end{align}
The regret bound above is obtained by setting the learning rates as follows: policy update rate $\alpha = \frac{\mu^2}{4G_1^2L}$, critic learning rates $(\beta,c_{\beta}) = \left(\frac{\lambda^2}{2},\lambda + \sqrt{\frac{1}{\lambda^2} - 1}\right)$ and NPG estimation rate $\gamma = \frac{\mu}{G_1^2}$. 

\begin{remark}[Scalability of $C_{\mathrm{hit}}$ and $C_{\mathrm{tar}}$]
The constants $C_{\mathrm{hit}}$ and $C_{\mathrm{tar}}$ reflect the difficulty of exploration under a given policy class and may scale with structural properties of the MDP such as state-space size or connectivity. Such dependence is inherent, as any algorithm must visit all states to avoid constant per-iteration regret, and even in ergodic MDPs the mixing time $t_{\mathrm{mix}}$ can scale with the state space. Our result extends existing ergodic-MDP guarantees to the more general unichain setting without degrading these constants. Specifically, while state-of-the-art actor–critic methods achieve complexity $O(t_{\mathrm{mix}}^{3}\sqrt{T})$~\cite{ganesh2024order}, our bound $O(C^{3}\sqrt{T})$ involves a constant $C$ that characterizes the Cesàro mixing time~\cite[Sec.~6.6]{levin2017markov}, which satisfies $C \le 7t_{\mathrm{mix}}$ for ergodic chains~\cite[Ex.~6.11]{levin2017markov}. In certain Markov chains, such as the biased random walk on the $n$-cycle~\cite[Ex.~24.2]{levin2017markov}, the Cesàro mixing time is only $O(n)$ versus the standard mixing time which is $\Theta(n^2)$, illustrating that averaging distributions over time accelerates convergence compared to single-step mixing.
\end{remark}

\begin{remark}[On the Unichain Assumption]\label{rem:ext}
To the best of our knowledge, no existing policy gradient methods (with general policy parametrization) have theoretical guarantees for average-reward, infinite-horizon MDPs beyond the unichain case. Our work is the first to establish such guarantees under the general unichain assumption, which is the weakest known condition under which the policy gradient theorem holds \cite{sutton1999policy}. While value-based methods admit guarantees beyond this setting, policy gradient approaches require fundamentally different analyses, making existing techniques inapplicable. A possible, though still conjectural, direction for relaxing the unichain assumption is to restrict policy search to unichain-inducing policies by maintaining strong exploration initially and annealing it as the algorithm converges. Notably, weakly communicating MDPs, the most general class solvable from a single stream of experience, always admit an optimal unichain policy \cite{puterman2014markov}.
\end{remark}

\end{document}